\documentclass{article}

\usepackage[preprint]{neurips_2022}

\usepackage[utf8]{inputenc} 
\usepackage[T1]{fontenc}    
\usepackage{hyperref}       
\usepackage{url}            
\usepackage{booktabs}       
\usepackage{amsfonts}       
\usepackage{nicefrac}       
\usepackage{microtype}      
\usepackage{xcolor}         
\usepackage{graphicx}
\usepackage{amsmath}
\usepackage{cleveref}
\usepackage{subcaption}
\usepackage{todonotes}
\usepackage{lipsum}

\usepackage{amsthm}

\theoremstyle{plain}
\newtheorem*{lemma}{Lemma}

\newcommand{\skl}{$\sqrt{\text{KL}}$}

\title{Scaling Laws for Reward Model Overoptimization}

\author{%
  Leo Gao \\
  OpenAI
  \And
  John Schulman \\
  OpenAI
  \And
  Jacob Hilton \\
  OpenAI

}

\begin{document}

\maketitle

\begin{abstract}
  In reinforcement learning from human feedback, it is common to optimize against a reward model trained to predict human preferences. Because the reward model is an imperfect proxy, optimizing its value too much can hinder ground truth performance, in accordance with Goodhart's law. This effect has been frequently observed, but not carefully measured due to the expense of collecting human preference data.  In this work, we use a synthetic setup in which a fixed ``gold-standard'' reward model plays the role of humans, providing labels used to train a proxy reward model. We study how the gold reward model score changes as we optimize against the proxy reward model using either reinforcement learning or best-of-$n$ sampling. We find that this relationship follows a different functional form depending on the method of optimization, and that in both cases its coefficients scale smoothly with the number of reward model parameters. We also study the effect on this relationship of the size of the reward model dataset, the number of reward model and policy parameters, and the coefficient of the KL penalty added to the reward in the reinforcement learning setup. We explore the implications of these empirical results for theoretical considerations in AI alignment.
  
\end{abstract}

\section{Introduction}

Goodhart's law is an adage that states, ``When a measure becomes a target, it ceases to be a good measure.'' In machine learning, this effect arises with proxy objectives provided by static learned models, such as discriminators and reward models. Optimizing too much against such a model eventually hinders the true objective, a phenomenon we refer to as \textit{overoptimization}. It is important to understand the size of this effect and how it scales, in order to predict how much a learned model can be safely optimized against. Moreover, studying this effect empirically could aid in the development of theoretical models of Goodhart's law for neural networks, which could be critical for avoiding dangerous misalignment of future AI systems.

In this work, we study overoptimization in the context of large language models fine-tuned as reward models trained to predict which of two options a human will prefer. Such reward models have been used to train language models to perform a variety of complex tasks that are hard to judge automatically, including summarization \citep{summarization}, question-answering \citep{webgpt, gophercite}, and general assistance \citep{instructgpt, anthropicrlhf1,glaese2022improving}. Typically, the reward model score is optimized using either policy gradient-based reinforcement learning or best-of-$n$ sampling, also known as rejection sampling or reranking. Overoptimization can occur with both methods, and we study both to better understand whether and how overoptimization behaves differently across both methods.

A major challenge in studying overoptimization in this context is the expense of collecting human preference labels. A large number of labels are required to accurately estimate overall preference probabilities, and this is exacerbated by small effect sizes and the need to take many measurements in order to fit scaling laws. To overcome this, we use a synthetic setup that is described in Section \ref{methodology}, in which labels are supplied by a ``gold-standard'' reward model (RM) instead of humans.

Our main results are empirically validated functional forms for the gold reward model scores $R$ as a function of the Kullback--Leibler divergence from the initial policy to the optimized policy $\text{KL} := D_{\text{KL}}\left(\pi\parallel\pi_{\text{init}}\right)$, which depends on the method of optimization used. This KL distance between the initial and optimized policies increases monotonically during during RL training (\cref{fig:kl_sweep_kl}), and can be computed analytically as a function of $n$ for BoN. Further, because it is a quadratic metric of distance \citep[Section 4.3]{anthropicrlhf1}, we will define $d := \sqrt{D_{\text{KL}}\left(\pi\parallel\pi_{\text{init}}\right)}$, and write our functional forms in terms of $d$.

We find empirically that for best-of-$n$ (BoN) sampling,
\[\boxed{R_{\text{bo}n}\left(d\right) = d\left(\alpha_{\text{bo}n}-\beta_{\text{bo$n$}}d\right)},\]

and for reinforcement learning,\footnote{We note that this form likely does not hold near the origin, as it has infinite slope there. We experimented with a number of different forms, but found worse fits and extrapolation. See \cref{rl_form_appx} for more details.}
\[\boxed{R_{\text{RL}}\left(d\right) = {d}\left(\alpha_{\text{RL}}-\beta_{\text{RL}}\log{d}\right)},\]

Here, $R(0) := 0$ by definition and $\alpha_{\text{RL}}$, $\beta_{\text{RL}}$, $\alpha_{\text{bo$n$}}$ and $\beta_{\text{bo$n$}}$ are parameters that may depend on the number of proxy reward model parameters, the size of the proxy reward model dataset, and so on. We see that these scaling laws make accurate predictions.

\begin{figure}[hp]
    \centering
    \begin{subfigure}{\linewidth}
    \includegraphics[width=\linewidth]{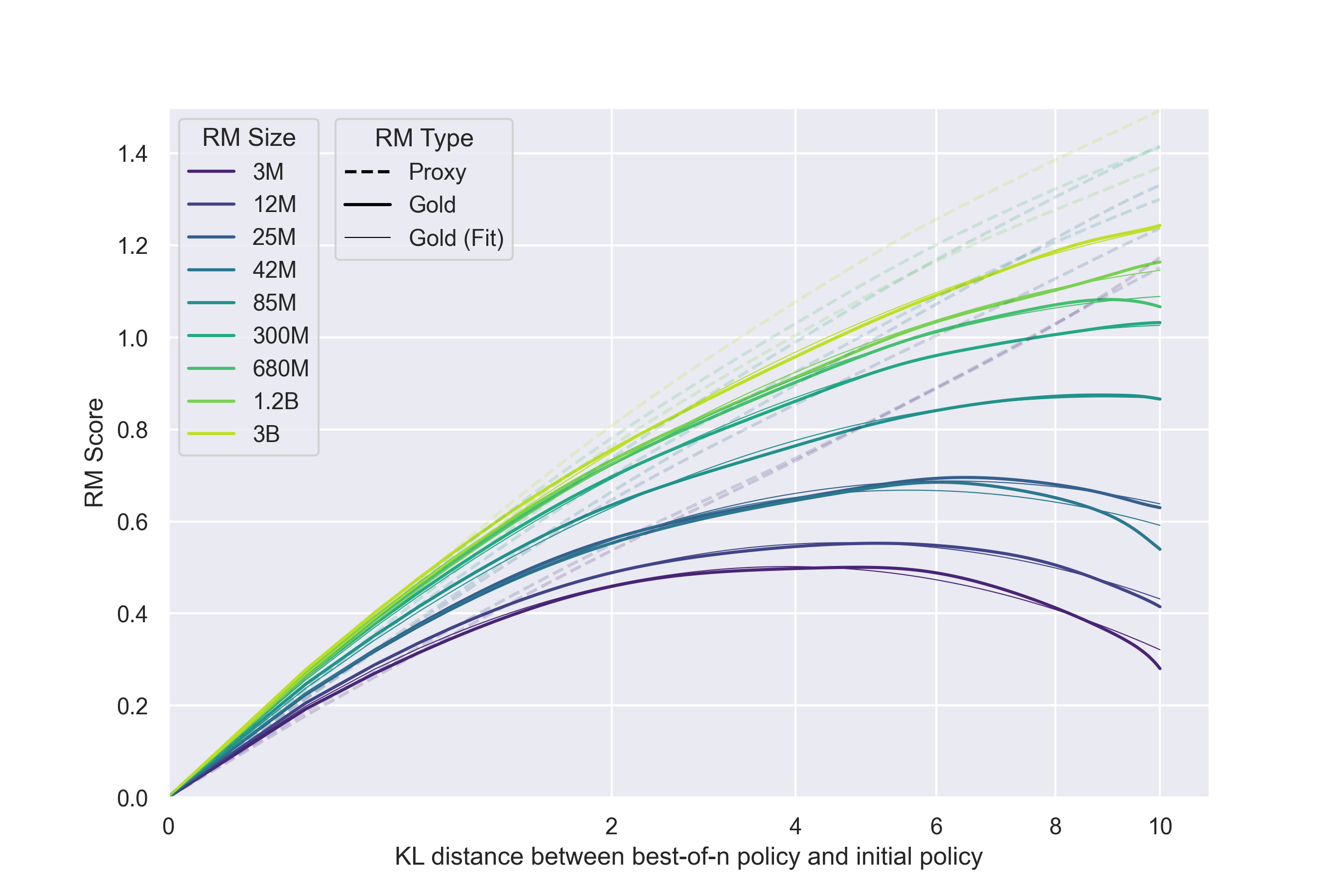}
    \caption{BoN}
    \label{fig:a_bon_rmsweep}
    \end{subfigure}
    \begin{subfigure}{\linewidth}
    \centering
    \includegraphics[width=\linewidth]{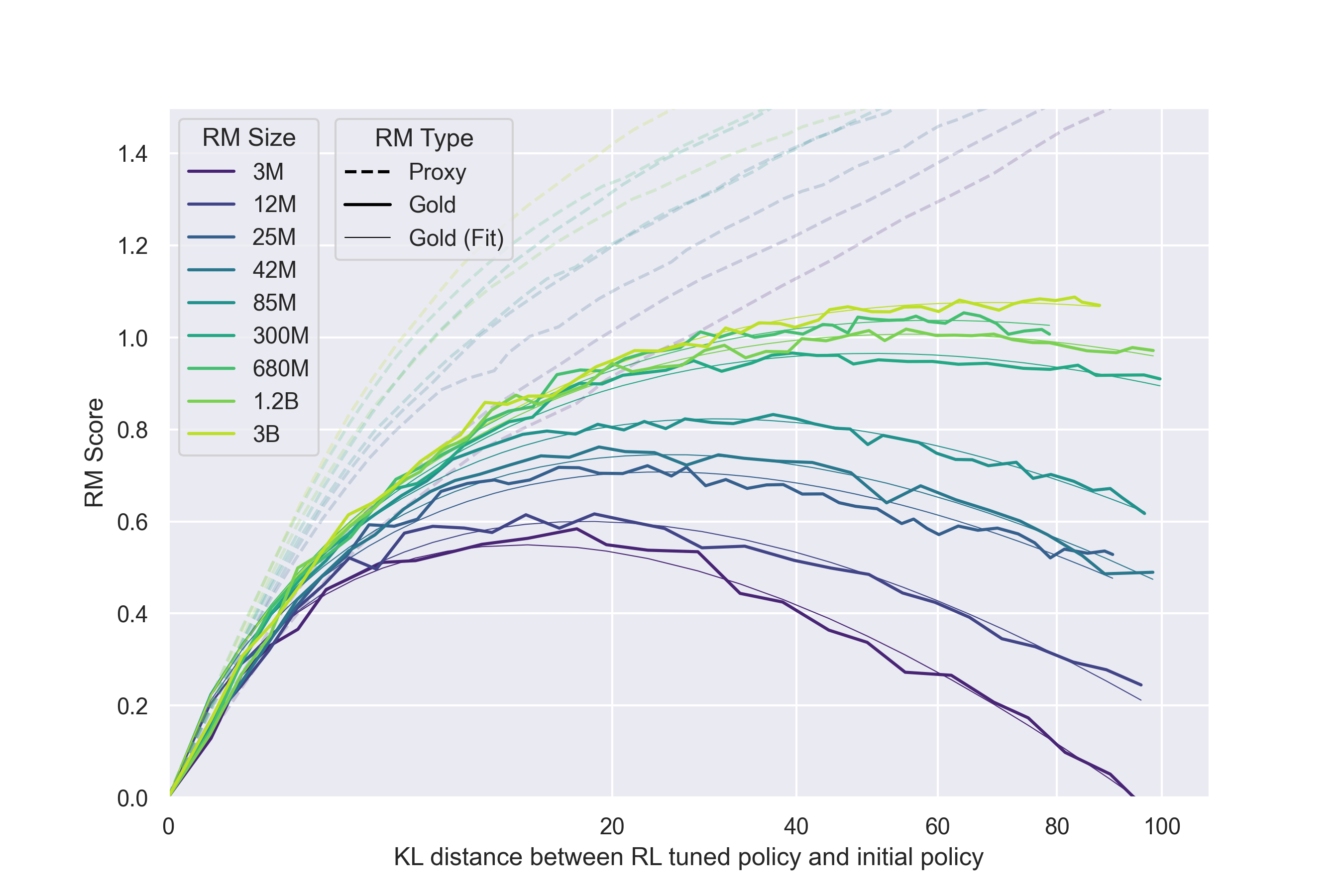}
    \caption{RL}
    \label{fig:a_rl_rmsweep}
    \end{subfigure}
    \caption{Reward model (RM) parameter size scaling experiments using the InstructGPT environment. Policy size is held constant (\texttt{1.2B}), while reward model size is varied. The x-axes have a square-root scale. Note that the plots have different x-axes. The gold reward represents the ground truth reward; we observe that when we optimize for a learned proxy of the gold reward, the gold reward initially increases and later decreases. We show that our functional forms fit this effect well.}
    \label{fig:rmsweep}
\end{figure}

We also find the following.
\begin{itemize}
    \item \textbf{RL versus best-of-$n$.} As a function of the KL divergence, reinforcement learning tends to be slower than best-of-$n$ sampling at both optimization and overoptimization. This suggests inadequacies with using KL to compare amount of (over)optimization across methods. However, the relationship between the proxy reward model score and the gold reward model score is similar for both methods.

    \item \textbf{Smooth coefficient scaling.} The $\alpha$ and $\beta$ coefficients in the BoN and RL functional forms vary smoothly with the number of proxy reward model parameters, following approximate logarithmic trends.\footnote{The coefficient $\alpha_{\text{RL}}$ in particular being nearly independent of RM parameter count.} This allows prediction of attained gold RM score.
    \item \textbf{Weak dependence on policy size.} While larger policies perform better overall and benefit less from optimization against an RM as measured by increase in gold reward, they lead to very similar amounts of overoptimization, as measured through the gap between the proxy and gold scores (which indicates the shortfall between predicted and actual reward), and KL distance at which the maximum gold RM score is attained.
    \item \textbf{KL penalty ineffectiveness.} In our reinforcement learning setup, using a KL penalty increases the proxy reward model score that can be achieved for a given KL divergence, but this does not correspond to a measurable improvement in the gold RM score--$\text{KL}_{\text{RL}}$ frontier. However, we note this result could be particularly sensitive to hyperparameters. 
\end{itemize}

Finally, we discuss the implications of these findings for Reinforcement Learning From Human Feedback (RLHF), existing models of Goodhart's law, and AI Alignment more broadly.

\section{Methodology}\label{methodology}

The setting used throughout this paper is the same as for InstructGPT \citep{instructgpt}. In our environment, the observations are text prompts and the policy is used to generate a response to the prompt. The prompts are drawn from a broad range of natural language instructions describing different language model tasks. Then, a learned RM is used to provide the reward signal for the response, which is used by either RL or BoN for optimization.

For all experiments, we use pretrained GPT-3 series language models as the initial checkpoint \citep{GPT3}. All initial policies are trained with supervised fine-tuning (SFT) on human-generated InstructGPT demonstrations \citep{instructgpt} for 2 epochs. All RMs also use the GPT-3 architecture but have an added scalar head to output the reward. 

The RL experiments use Proximal Policy Optimization (PPO) \citep{schulman2017proximal}. KL penalty for all RL experiments is set to 0 except for in \cref{kl_sweep_sec}. See \cref{hparams} for all other hyperparameters. We mostly use defaults for the PPO hyperparameters; thus, it is possible that there exist different trends for other hyperparameter configurations.

In BoN, we generate $n$ trajectories for the policy and use the reward model to pick the one with the highest proxy RM score. We use the unbiased estimator from \citet[Appendix I]{webgpt} to compute all of the gold and proxy scores for intermediate $n$ between 1 and the maximum $n$ with lower variance and more efficiently than the naive estimator of randomly sampling $n$ samples with replacement repeatedly and taking the mean of the maximum gold and proxy RM scores. The KL distances for BoN are computed analytically: $\text{KL}_{\text{bo}n} = \log n - \frac{n-1}{n}$ \citep[Appendix G.3]{summarization}. 

\subsection{Synthetic Data Setup}

Because getting a ground truth gold reward signal from human labellers is expensive, we instead use a synthetic task where the ground truth is defined to be the output of a particular large ``gold'' RM. The 6B reward model from \citet{instructgpt} is used as the gold RM, and our proxy RMs vary from 3M to 3B parameters\footnote{We originally trained two additional RMs smaller than 3M parameters, which achieved near-chance accuracy and were off-trend, and so were excluded.}. This synthetic gold reward is used to label pairs of rollouts from the policy given the same prompt to create synthetic RM training data. The synthetic comparisons are created deterministically by always marking the trajectory with the higher gold RM score as preferred.\footnote{We had experimented with sampling for creating labels, but observed noisier results.} We generate 100,000 synthetic comparisons and reserve 10\% of these as a held out test set for computing the validation loss of RMs.

See \cref{fig:setup_dia} for a diagram of the synthetic setup.

\subsection{Recalibration}\label{norm_and_cal}

The RM scores are translation-invariant, so to ensure comparability across different reward models, we recenter each RM such that the average reward of the initial policy is 0. We also unit normalize the variance of the gold RM scores.\footnote{We later decided this was unnecessary but decided not to change it.} Because our hard thresholding synthetic data setup produces labels that are miscalibrated (since they do not incorporate the gold RM's confidence), we recalibrate the proxy RMs by rescaling the logits to minimize cross-entropy loss using a validation set of soft labels. All renormalization and recalibration is applied after the experiments; this does not affect BoN at all, and likely has no impact on RL because Adam is loss scale invariant, though it is possible that there are slight differences due to algorithmic details.

\begin{figure}[t]
    \centering
    \includegraphics[width=\linewidth]{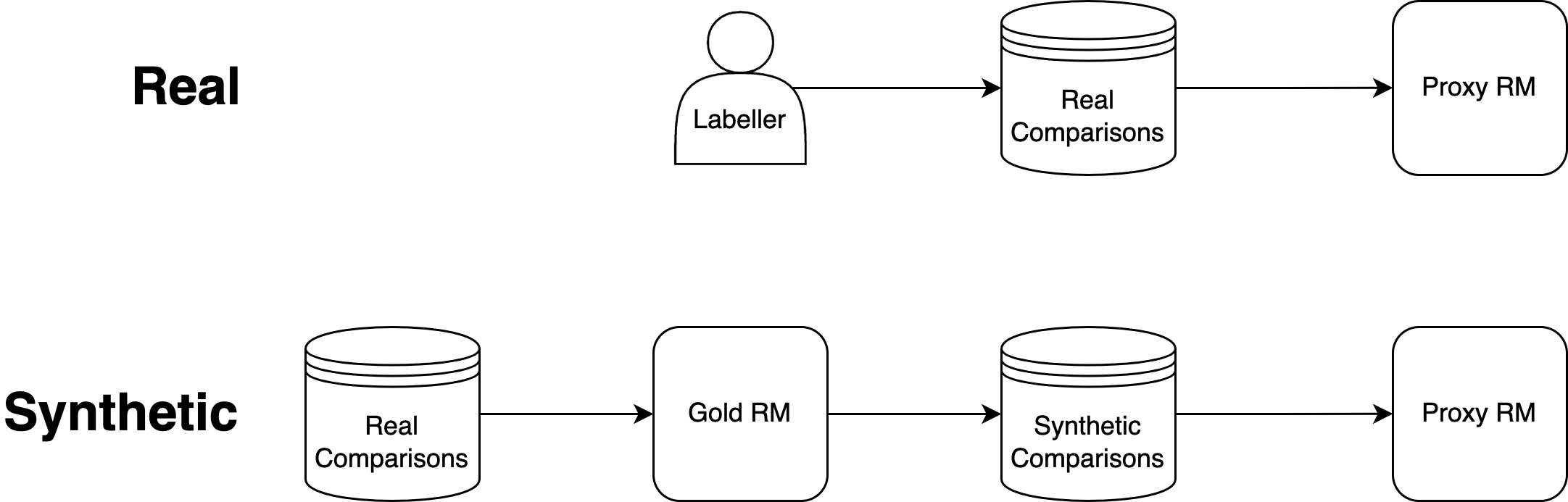}
    \caption{Diagram of the real and synthetic RM training setups. Human labellers generate comparison data. In the real RLHF setting, this data is used to train a proxy RM that is optimized by RL/BoN. In our synthetic setting, we instead use a ``Gold RM'' as our ground truth. In both settings, the proxy RM is a proxy for the ground truth process generating the labels (either the human or gold RM).}
    \label{fig:setup_dia}
\end{figure}
\section{Results}

\subsection{Fitting and validating functional forms}\label{validating_closedforms}

We chose our functional forms through experimentation with all RM data and parameter scaling curves in the remainder of this paper.

The BoN functional form was hypothesized using data up to $n = 1000$. In order to validate the functional forms, we performed a BoN experiment with up to $n = 60,000$ (KL $\approx$ 10 nats), after only having seen data up to $n = 1,000$ (KL $\approx$ 6 nats). As this experiment was conducted after the functional form was hypothesized based on data up to 6 nats, this was a true advance prediction.

We also test extrapolation of the BoN and RL functional forms from low KLs to to unseen larger KLs; see \cref{fig:rmsweep_extrap} for details.

We also attempted to model the proxy scores but were unable to obtain a satisfactory fit. For BoN, despite visual similarity, a linear fit ($d\alpha_{\text{bo}n}$) did not work well (\cref{fig:a_bon_rmsweep_proxy_fit}). The predictions for RL and BoN are not as easily modelled as the gold score predictions. We leave a better understanding of the proxy RM score behavior to future work.

\subsection{Scaling with RM Parameter Count}\label{scaling_rm}

We hold policy size (1.2B) and data size (90,000) constant (\cref{fig:rmsweep}). We observe that for the gold RM scores, $\alpha_{\text{bo}n}$ and $\beta_{\text{bo}n}$ change smoothly with RM size (\cref{fig:a_bon_rmsweep_fit_alpha,fig:a_bon_rmsweep_fit_beta}). For RL, we find that we can hold $\alpha_{\text{RL}}$ constant across all RM sizes, resulting in a clean scaling curve for $\beta_{RL}$ (\cref{fig:a_rl_rmsweep_fit_beta}). These scaling laws allow us to predict properties of training runs; for instance, we can also predict the peak gold RM scores for different RM sizes (\cref{fig:a_max_gold_score}). 

When modelled using the same functional forms as the respective gold scores, the proxy score fits have much lower values of $\beta_{\text{bo}n}$. We also see smooth scaling in the proxy score's $\alpha_{\text{bo}n}$ and $\beta_{\text{bo}n}$. However, for the reasons in \cref{validating_closedforms}, we are less confident about these fits. For both BoN and RL, we observe systematic underestimates of the proxy reward model when extrapolated to higher KLs. Both appear to eventually grow roughly linearly in \skl, as in \citet{anthropicrlhf1}.

\begin{figure}[h] 
    \centering
    \begin{subfigure}[h]{0.3\linewidth}
    \includegraphics[width=\linewidth]{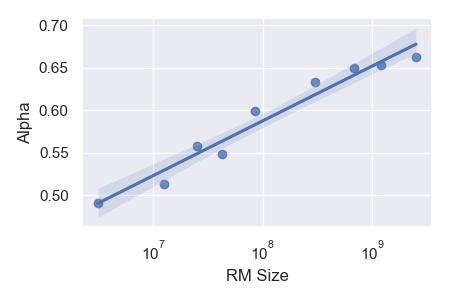}
    \caption{$\alpha_{\text{bo}n}$}
    \label{fig:a_bon_rmsweep_fit_alpha}
    \end{subfigure}
    \begin{subfigure}[h]{0.3\linewidth}
    \centering
    \includegraphics[width=\linewidth]{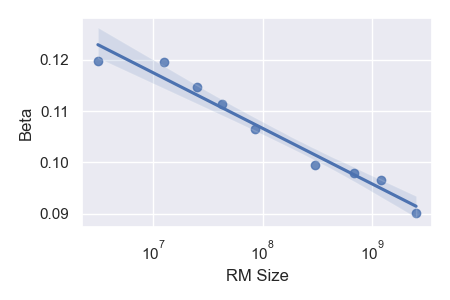}
    \caption{$\beta_{\text{bo}n}$}
    \label{fig:a_bon_rmsweep_fit_beta}
    \end{subfigure}
    \begin{subfigure}[h]{0.3\linewidth}
    \centering
    \includegraphics[width=\linewidth]{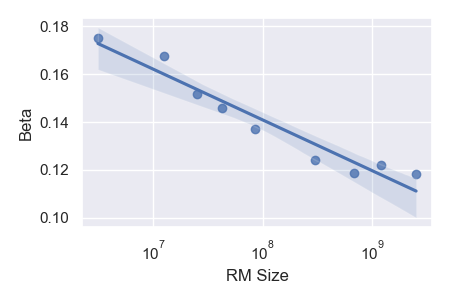}
    \caption{$\beta_{\text{RL}}$}
    \label{fig:a_rl_rmsweep_fit_beta}
    \end{subfigure}
    \caption{The values of $\alpha_{\text{bo}n}$, $\beta_{\text{bo}n}$ and $\beta_{\text{RL}}$ in the BoN and RL overoptimization scaling laws for both proxy (dashed line) and gold (solid line) rewards as they scale with parameter count.}
    \label{fig:rmsweep_fit}
\end{figure}

\subsection{Scaling with RM Data Size}\label{data_scaling}

We hold RM size constant (12M) and sweep RM data size for both RL and BoN.\footnote{For BoN, we actually sweep all combinations of RM size and data size; see \cref{fig:a_datasweep_data_param_tradeoff}. For a version of \cref{fig:a_bon_datasweep} against a 3B RM, see \cref{fig:a_bon_datasweep_3brm}.}. Overall, the results are consistent with intuition: more data leads to better gold scores and less goodharting. The scaling of $\alpha$ and $\beta$ with data size are not as cleanly described as for RM size scaling (\cref{fig:a_bon_datasweep_alpha}, \cref{fig:a_bon_datasweep_beta}). 

\begin{figure}[h]
    \centering
    \begin{subfigure}[h]{0.49\linewidth}
    \includegraphics[width=\linewidth]{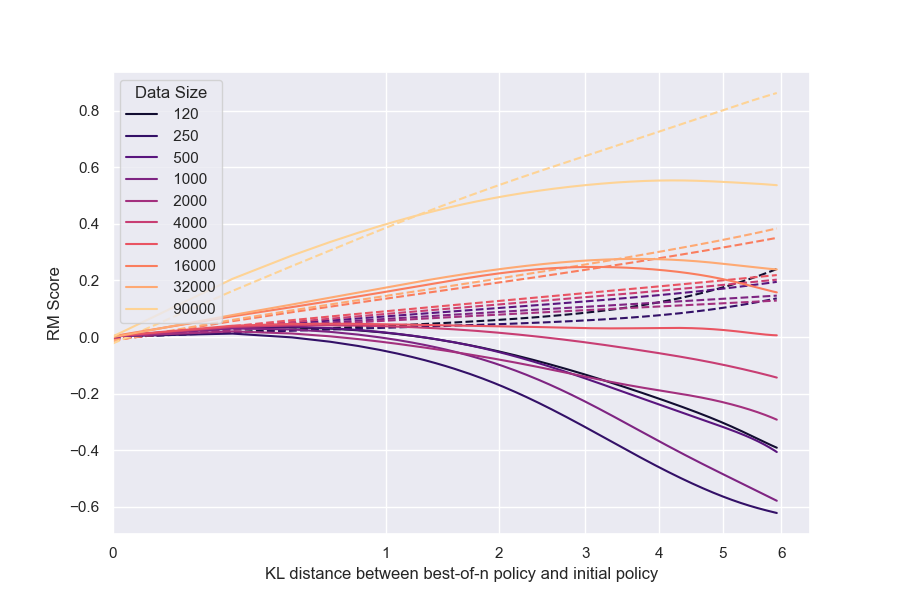}
    \caption{BoN}
    \label{fig:a_bon_datasweep}
    \end{subfigure}
    \hfill
    \begin{subfigure}[h]{0.49\linewidth}
    \centering
    \includegraphics[width=\linewidth]{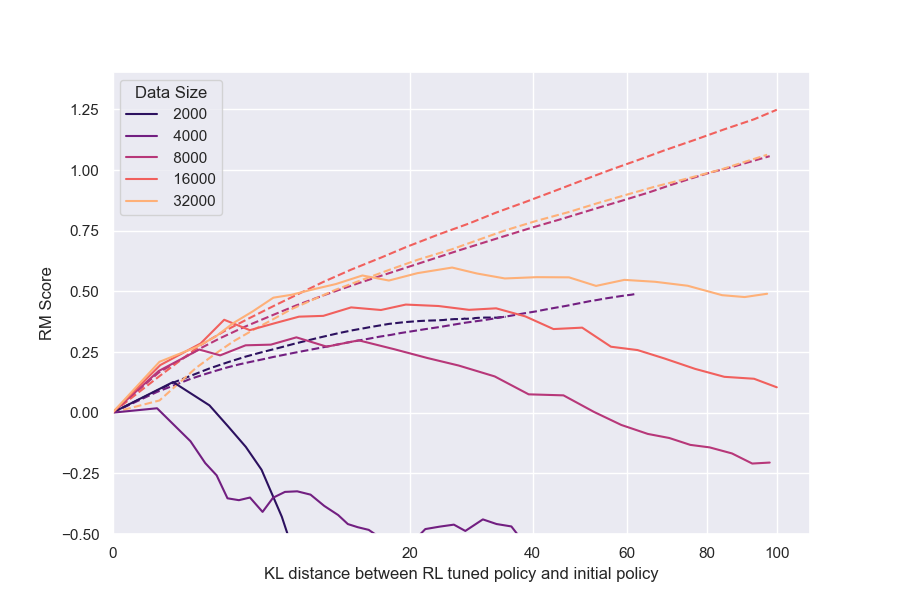}
    \caption{RL} 
    \label{fig:a_rl_datasweep}
    \end{subfigure}
    \caption{RM data scaling experiments. RM size is held constant (12M), while RM data is varied. The x-axis has a square root scale.  Note that the plots have different axes. Dotted lines indicate proxy rewards, solid lines indicate gold rewards.}
    \label{fig:datasweep}
\end{figure}

For all RM sizes, we observe that for amounts of data less than around 2,000 comparisons\footnote{To test the hypothesis that some minimum number of RM finetuning steps is needed, we control for the number of SGD steps by running multiple epochs and observe that running 4 epochs instead of 1 yields no change in gold score whatsoever, whereas 1 epoch of 4 times as much data performs substantially better (\cref{fig:a_epochs_control}).}, there is very little improvement over near-chance loss (\Cref{fig:a_bon_datasweep_rmloss}). This is also reflected in gold scores after optimization (\cref{fig:a_bon_datasweep_final}). After this threshold, all models improve with more data, though larger RMs generally improve faster. Interestingly, although larger RMs result in better gold scores overall, they do not appear to have this critical threshold substantially earlier than smaller models.\footnote{This result contradicts some other internal findings; thus, it is possible that this is an artifact of this particular setup.}

We hypothesized that two RMs of equal validation loss would achieve the same robustness against optimization, regardless of the combination of RM size and RM data size. Our results provide some weak evidence for this hypothesis (\cref{fig:a_bon_datasweep_final_rmloss}).

\begin{figure}[h]
    \centering
    \begin{minipage}{0.49\linewidth}
        \includegraphics[width=\linewidth]{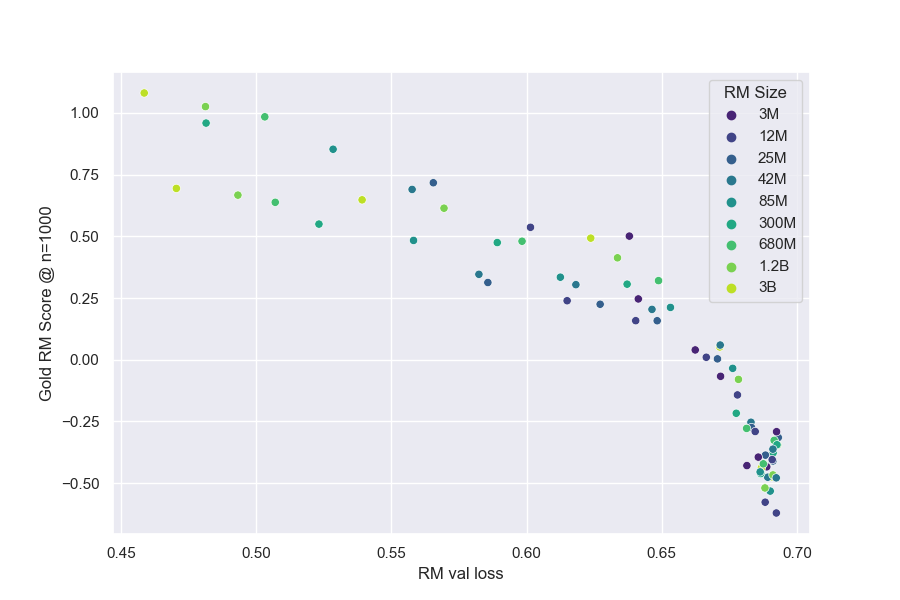}
        \caption{RM validation loss vs BoN RM score @ n=1000. Most points in this figure are already averaged over multiple seeds.}
        \label{fig:a_bon_datasweep_final_rmloss}
    \end{minipage}
    \hfill
    \begin{minipage}{0.49\linewidth}
        \includegraphics[width=\linewidth]{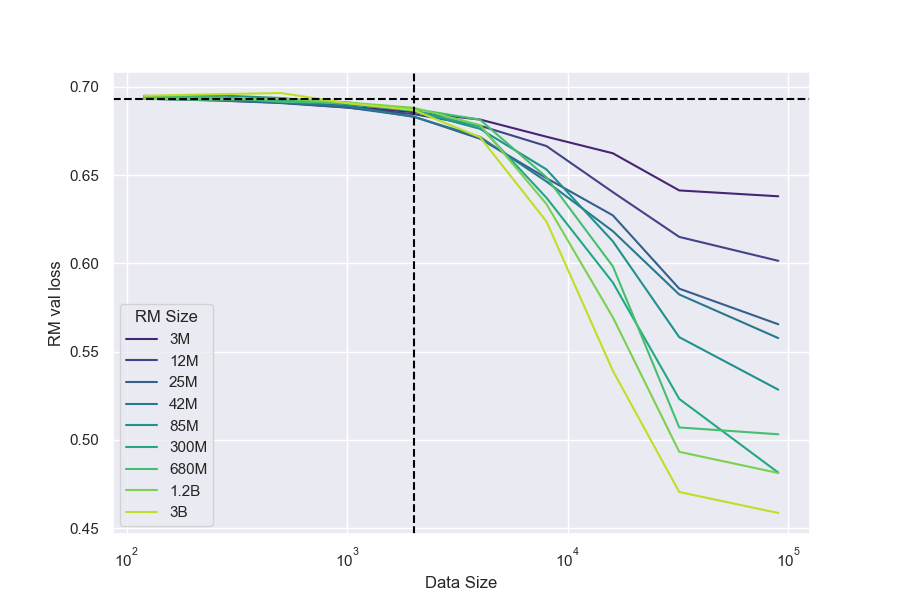}
        \caption{RM losses, broken down by data size and RM size}
        \label{fig:a_bon_datasweep_rmloss}
    \end{minipage}
\end{figure}

\subsection{Scaling with Policy Size}\label{policysweep}

\begin{figure}[t]
    \centering
    \begin{subfigure}{0.49\linewidth}
    \includegraphics[width=\linewidth]{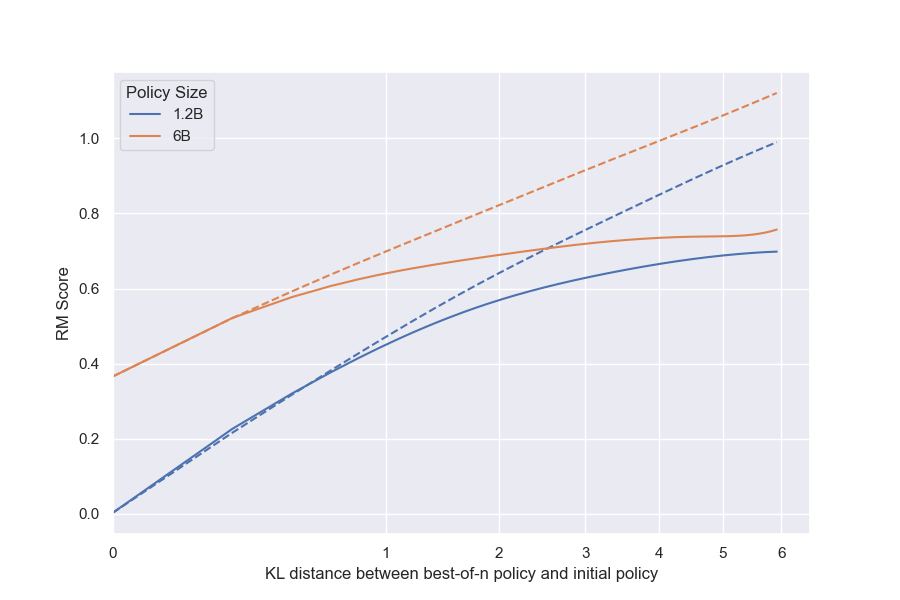}
    \caption{BoN}
    \label{fig:a_bon_policy_6b}
    \end{subfigure}
    \begin{subfigure}{0.49\linewidth}
    \centering
    \includegraphics[width=\linewidth]{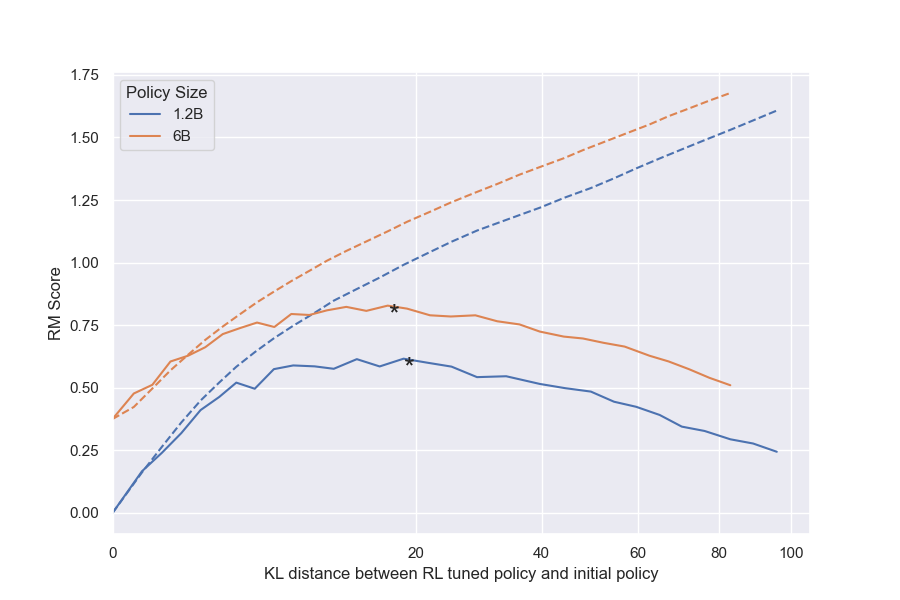}
    \caption{RL}
    \label{fig:a_rl_policy_6b}
    \end{subfigure}
    \caption{Policy scaling experiments. RM size is held constant (12M), while policy size is varied. The x-axis has a square root scale. Note that the plots have different axes. Dotted lines indicate proxy rewards, solid lines indicate gold rewards. The asterisks in the RL plot indicate the max gold score for each policy size.}
    \label{fig:policysweep}
\end{figure}

We briefly explore the impact of policy size by holding the RM size constant (12M) and evaluating two different policy sizes. We also perform the same experiment with a different RM size (3B), observing similar results (\cref{fig:a_rl_policy_6b_3brm}). 

\paragraph{Larger policies see less benefit from optimization against an RM, but don't overoptimize more.} We observe that the 6B policy run has a smaller difference between its initial and peak gold reward model scores than the 1.2B policy run. This is most visible in the BoN plot (\cref{fig:a_bon_policy_6b}).\footnote{For a version of the RL plot (\cref{fig:a_rl_policy_6b}) with all runs starting at 0, see \cref{fig:a_rl_policy_6b_from0}.} However, while we might expect that a larger policy overoptimizes substantially faster, contrary to intuition, we find that both gold scores peak at almost the same KL. In fact, the gap between the proxy and gold scores is almost the same between the two policy sizes (\cref{fig:a_rl_policy_6b_gap}). We can interpret this gap, the shortfall between the predicted and actual rewards, as being indicative of the extent to which the proxy RM is exploited. We discuss this result further in \cref{policysweep_implications}.

\subsection{RL vs BoN}\label{rl_vs_bon}

\textit{A priori}, we might expect reinforcement learning via PPO \citep{schulman2017proximal} and best-of-n to apply optimization in very different ways. As such, we ask whether this difference in optimization results in different overoptimization characteristics. Similarities would potentially indicate candidates for further study in gaining a more fundamental understanding of overoptimization in general, and differences opportunities for better optimization algorithms. We note the following:

\paragraph{RL is far less KL-efficient than BoN.} Viewing KL distance as a resource to be spent, we observe that RL "consumes" far more KL than BoN. This means that {both} optimization and overoptimization require more KL to occur with RL. Intuitively, BoN searches very locally around the initial policy, and thus $\text{KL}_{\text{bo}n}$ increases with roughly $\log(n)$. For RL on the other hand, each step modifies the policy from the policy of the previous step---KL increases approximately quadratically with step in the absence of KL penalty (\Cref{fig:a_rl_rmsweep_step_kl}, \Cref{fig:kl_sweep_kl}). An implication of this result is that KL distance is an inadequate metric for quantity of (over)optimization; we discuss this further in \cref{kl_measure}. 

\paragraph{When looking at proxy vs gold RM scores, BoN and RL look more similar.} The proxy RM score is another possible metric for quantity of optimization, because it is the value that is being directly optimized for. Using it as the metric of optimization leads to significantly more analogy between RL and BoN than KL distance does. However, we do observe that RL initially has a larger proxy-gold gap (i.e requires more proxy RM increase to match BoN), but then peaks at a higher gold RM score than BoN (\cref{fig:a_train_rm_vs_gold_rm}).

\begin{figure}[h]
    \centering
    \includegraphics[width=0.8\linewidth]{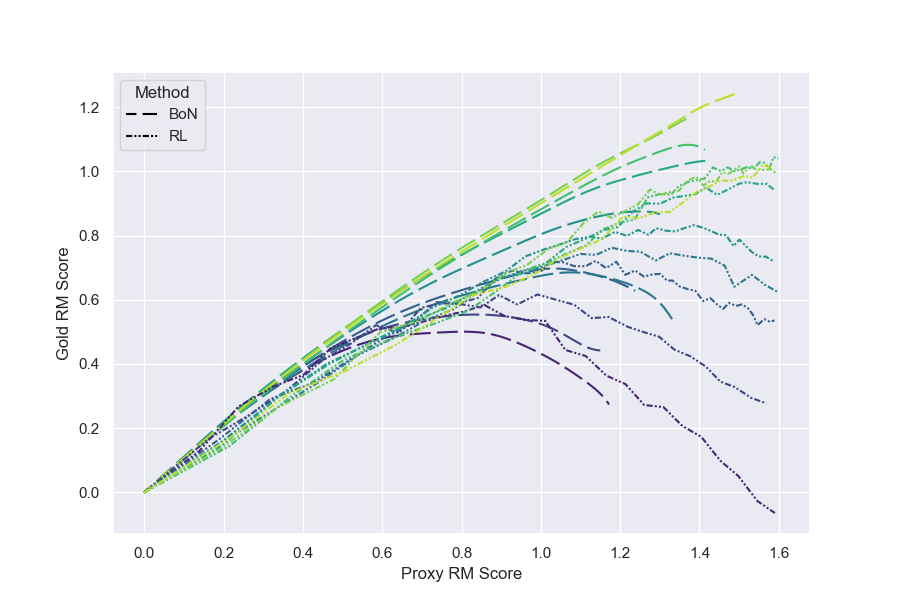}
    \caption{Proxy vs gold RM score for both BoN and RL. RL curves are truncated to a proxy RM score of 1.6 for readability.}
    \label{fig:a_train_rm_vs_gold_rm}
\end{figure}

\subsection{Effect of KL Penalty}\label{kl_sweep_sec}

We observe in our setting that when varying the KL penalty for RL, the gold RM scores depend only on the KL distance of the policy $\text{KL}_{\text{RL}}$ (\Cref{fig:kl_sweep}). The KL penalty only causes the gold RM score to converge earlier, but does not affect the $\text{KL}_{\text{RL}}$-gold reward frontier, and so the effect of the penalty on the gold score is akin to early stopping (\Cref{fig:kl_sweep_kl}). However, we have seen some evidence that this result could be particularly sensitive to hyperparameters. 

Because we observe that using KL penalty has a strictly larger proxy-gold gap, we set KL penalty to 0 for all other RL experiments in this paper.

It is important to note that PPO's surrogate objective incorporates an implicit penalty on $D_{\text{KL}}\left(\pi_{\text{old}}\parallel\pi\right)$, where $\pi_{\text{old}}$ is a recent policy (not the initial policy) \citep{schulman2017proximal}. This penalty is used to control how fast the policy changes, but also has an indirect effect on the KL we study here, $D_{\text{KL}}\left(\pi\parallel\pi_{\text{init}}\right)$, causing it to grow much more slowly (providing the implementation is well-tuned). We do not know why this indirect effect appears to lead to less overoptimization than an explicit KL penalty.

\begin{figure}[h]
    \centering
    \includegraphics[width=0.8\textwidth]{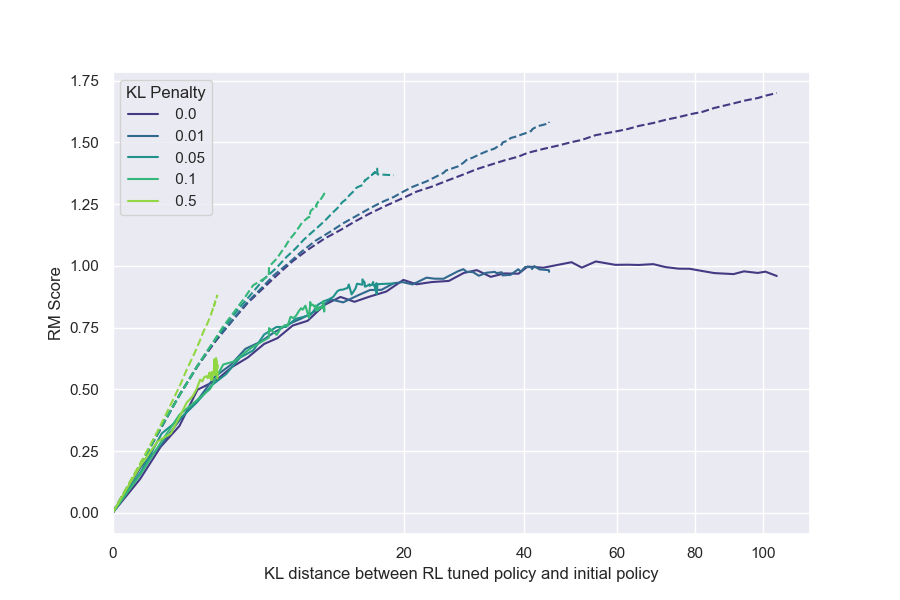}
    \caption{RL experiments with various KL penalties. Policy size (1.2B) and RM size (1.2B) are held constant. Dotted lines indicate proxy rewards, solid lines indicate gold rewards. We observe the effect of the KL penalty on the gold score as being equivalent to early stopping.}
    \label{fig:kl_sweep}
\end{figure}

\section{Discussion}

\subsection{KL as a measure of amount of optimization}\label{kl_measure}

For any given fixed optimization method, KL yields clean scaling trends, such as the ones observed in \cref{scaling_rm}, and consistent peak gold RM score KLs as in \cref{policysweep}. However, because it's clear that different methods of optimization spend KL very differently (\cref{rl_vs_bon}), it should not be used to compare the amount of optimization between different optimization algorithms. There exist pertubations to a policy that are orthogonal to the reward signal that would result in increases in KL that do not increase either gold or proxy reward; conversely, extremely small but well targeted perturbations could substantially change the behavior of the policy within a small KL budget.

\subsection{Relation to Goodhart Taxonomy}\label{implications_optim}

One useful taxonomy for various Goodhart effects is presented in \citet{manheim2018categorizing}, categorizing Goodhart's Law into 4 categories: Regressional, Extremal, Causal, and Adversarial. In this section, we discuss our results in the framework of this taxonomy.

\subsubsection{Regressional Goodhart}\label{regressional_goodharting}

Regressional Goodhart occurs when our proxy RMs depend on features with noise. The simplest toy example of this is a proxy reward $\hat X$ which is exactly equal to the gold reward $X$ plus some independent noise $Z$. When optimizing against this proxy, some amount of optimization power will go to selecting for noise, leading to a gold reward less than predicted by the proxy. 

More formally, for independent absolutely continuous random variables $X$ and $Z$ with $X$ normally distributed and either
(a) $Z$ normally distributed or (b) $\left|Z-\mathbb E\left[Z\right]\right|<\delta$ for some $\delta>0$, this model predicts a gold reward that is:

\begin{equation}\label{eq:reg_goodhart}
    \mathbb E[X \mid \hat X = \hat x] = \mathbb E[X] + \left(\hat x - \mathbb E[X] - \mathbb E[Z]\right) \frac{\mathrm{Var}(X)}{\mathrm{Var}(X) + \mathrm{Var}(Z)}+\varepsilon
\end{equation}

where $\varepsilon=0$ in case (a) and $\varepsilon=o\left(\mathrm{Var}\left(Z\right)\right)$ as $\delta\to 0$ in case (b). See \cref{reg_goodhart_proof} for the proof.

Intuitively, we can interpret \cref{eq:reg_goodhart} as stating that the optimization power expended is divided between optimizing the gold reward and selecting on the noise    proportional to their variances. This also implies that if this is the only kind of Goodhart present, the gold reward must always increase monotonically with the proxy reward; as we observe nonmonotonic behavior (\cref{fig:a_train_rm_vs_gold_rm}), there must be either noise distributions violating these assumptions or other kinds of Goodhart at play.

This result lends itself to an interpretation of the $\alpha$ term in the RL and BoN gold score scaling laws: since for both RL and BoN the proxy scores are roughly linear in \skl, the difference in the slope of the proxy score and the linear component of the gold score (i.e the $\alpha$ term) can be interpreted as the amount of regressional Goodhart occurring. 

\subsubsection{Extremal Goodhart}\label{extremal_goodhart}

We can think of out of distribution failures of the RM as an instance of extremal Goodhart. As we optimize against the proxy RM, the distribution of our samples shifts out of the training distribution of the RM, and thus the relation between the proxy and gold scores weakens. For instance, suppose in the training distribution a feature like answer length always indicates a higher quality answer, and thus the proxy RM infers that longer answers are always better, even though at some point outside the training distribution, selecting on longer answers no longer improves quality.\footnote{Optimized policies producing very long answers even when a short answer would be preferred is a real issue that we have observed in other experiments in the InstructGPT setting.}

We can also think of this as the proxy failing to depend on relevant features; this failure bears resemblance to the setting considered in \citet{zhuang2020consequences}, where a failure of the proxy to consider all features, under certain conditions, leads to overoptimization with unbounded loss of utility regardless of optimization method.

We expect extremal Goodharting to be primarily responsible for the nonmonotonicity of the gold RM scores in this paper, and is mostly responsible for the $\beta$ term, which in the limit of optimization, results in an unbounded loss of utility. This lends a natural interpretation to the smooth decrease in $\beta$ for both BoN and RL with increased RM size as smooth improvements in model robustness
(\cref{fig:rmsweep_fit}).

\subsubsection{Causal Goodhart}

We can think of causal Goodhart as being a generalization of regressional Goodhart: there may exist correlations between features and gold score where the causal structure of the problem is such that selecting on the feature does not increase the gold score. For instance, suppose answer length is correlated with quality due to some other common cause (say, informativeness); then, the proxy RM may learn to use answer length as a feature, and when we select against the proxy we get longer answers that do not increase on actual quality.\footnote{We can think of noise as a particular case of this where the independent noise is correlated with signal+noise, but of course there is no causal relation between signal and noise.} In our experiments, we would observe causal Goodhart as behaving similarly to regressional Goodhart. 

\subsubsection{Adversarial Goodhart}

Adversarial Goodhart occurs when the policy actively manipulates the proxy. We do not expect the effects of adversarial Goodhart to be captured in this work, as the models involved are not powerful enough to implement adversarial strategies. However, given the constant improvement of ML capabilities, it is entirely plausible that ML systems will one day become capable enough to do so \citep{hubinger2019rlo}. When this occurs, the scaling laws observed in this paper may break down. Thus, we advise caution when using these results for extrapolation.

    \subsection{Implications for iterated RLHF}\label{iterated_rlhf}

When conducting reinforcement learning from human feedback, it is preferable to use an online setup, in which fresh human feedback data is periodically used to train a new reward model, to mitigate overoptimization \citep{anthropicrlhf1}. Our scaling law allows us to analyze the effect of this iterative approach under some simplifying assumptions. We assume firstly that the scaling coefficients $\alpha_{\text{RL}}$ and $\beta_{\text{RL}}$ remain constant across iterations, and secondly that the distance $d=\sqrt{\text{KL}}$ is additive across iterations (because of how KL appears to grow empirically as in Figure \ref{fig:kl_sweep_kl}). Under these assumptions, the final gold reward model score after $k$ iterations each covering a distance $d/k$ is given by
\[R_{\text{RL}}\left(d\right) = {d}\left(\alpha_{\text{RL}}-\beta_{\text{RL}}\log\left({d}\right)+\beta_{\text{RL}}\log\left(k\right)\right).\]

Two interesting observations follow from this. Firstly, the iterative approach does not affect any Goodharting captured by the $\alpha_{\text{RL}}$ term (such as regressional Goodharting, as discussed in Section \ref{regressional_goodharting}). Secondly, the effect of the iterative approach is to increase the final gold RM score by an amount proportional to both $d$ and $\log\left(k\right)$, namely
\[\beta_{\text{RL}}d\log\left(k\right).\]
Note that this result can only hold up to some maximum value of $k$, and we expect our scaling law to break down below some minimum distance. Further research is required to determine what this minimum is, as well as to what extent our simplifying assumptions hold in practice.

\subsection{Policy size independence}\label{policysweep_implications}

Our observation that larger SFT policies seem to exhibit the same amount of overoptimization during RL implies that larger policies do not increase the amount of optimization power applied to the RM or learn faster, even though they start out with higher performance on the gold score. While it is expected that larger policies have less to gain from optimizing against the same RM, we might also expect the gold score to peak at a substantially earlier KL distance, analogous to what we see when we scale the RM size (\cref{scaling_rm}), or for larger policies to more efficiently utilize the same number of RL feedback steps (\cref{data_scaling})\footnote{It is also not the case that the 6B policy run has higher KL distance for the same number of RL steps; in fact, we observe that it has \textit{lower} KL distance for the same number of steps (\cref{fig:a_rl_policy_6b_step_kl})}.

One possible hypothesis is that, because RLHF can be viewed as Bayesian inference from the prior of the initial policy \citep{korbak2022rl}\footnote{The result of \citet{korbak2022rl} concerns varying KL penalties rather than KL distances with no KL penalty, but as we observe in \cref{kl_sweep_sec}, this is equivalent on our setting.}, increases in policy size are only improving the modelling accuracy of the human demonstration distribution.

\subsection{Limitations and Future Work}\label{limitations}

In addition to the overoptimization studied in this paper (due to the mismatch between the reward model and the ground truth labels), there exists another source of overoptimization due to mismatch between the ground truth labels and the actual human intent. This contains issues ranging from the mundane, such as labellers choosing options that only \textit{appear} to match their intent\footnote{For instance, the example of a robotic hand learning from human feedback to only \textit{appear} to grasp a ball, presented in \url{https://openai.com/blog/deep-reinforcement-learning-from-human-preferences/} \citep{christiano2017deep}}, to substantially more philosophically fraught issues \citep{NEURIPS2018_d89a66c7,sunstein2001predictably}. The main limitation of this work is that this additional source of overoptimization is not captured in the setting of this paper. See \cref{related_work} for discussion of related work in alignment.

Some additional limitations and future directions include:

\begin{itemize}
    \item \textbf{Validating these results on other environments and experimental setups.} While the experiments in this paper all use the InstructGPT environment, the main value of these results lies in the extent to which they reflect general phenomema. Confirming whether these results generalize to other settings would be extremely valuable to that end.\footnote{In the course of our experiments, we observed visually similar results on the WebGPT environment \citep{webgpt}.} 
    \item \textbf{Validating the synthetic setting.} The synthetic setting might not transfer to real world settings, for instance because there is substantial correlation between RMs. 
    \item \textbf{Investigating methods for making RMs more robust to optimization.} While there has been prior work in this direction (see \cref{related_work}), there is still much work to be done in systematically investigating ways to make RMs more robust.
    \item \textbf{Exploring other forms of optimization and categorizing their differences.} While this work focuses exclusively on BoN and RL there are other ways of applying optimization pressure against a model of a reward signal, either implicit or explicit. This includes GeDi-like steering, Decision Transformers\footnote{One could consider measuring the actual achieved ground truth/gold score achieved for each "proxy" score conditioned on, a la \cref{fig:a_train_rm_vs_gold_rm}, as testing the implicit reward-behavior mapping encoded by the model.}, variants of BoN like beam search, and other RL algorithms.
    \item \textbf{Better understanding the functional form of proxy RM scores.} In our modeling, we find that the proxy RM scores are more difficult to predict for both BoN and RL (\cref{scaling_rm}). While they seem to have a major linear component, there is sufficient variation that fitting a linear regression is not very good at predicting extrapolated proxy RM scores.
    \item \textbf{Exploring adversarial Goodhart empirically.} In this work we deal with systems not powerful enough to cause adversarial Goodhart. However, it is plausible that adversarial Goodhart is especially important, or is associated with phase changes that break the trends seen in this paper.
    \item \textbf{Exploring scaling with policy size in more detail.} Our exploration of policy size scaling in this paper was limited to only two policy sizes. It is possible that there exist trends not seen in our exploration when considering the policy size more carefully.
    \item \textbf{Exploring multi-iteration RLHF.} In particular, checking for deviations from the assumptions of \cref{iterated_rlhf}.
\end{itemize}

We hope this paper leads to future work further bridging conceptual and empirical alignment research.

\section{Related Work}\label{related_work}

Goodhart's Law in its modern formulation was first introduced in \citet{hoskin1996}, with many of the key ideas introduced in prior works \citep{campbell1969reforms,goodhart1975problems}. Many approaches have been proposed for reducing overoptimization in general \citep{taylor2016quantilizers,everitt2017reinforcement}, as well as in RMs \citep{gleave2022uncertainty}, including within the field of adversarial robustness \citep{chakraborty2018adversarial}. Overoptimization of reward models can be viewed as a special case of specification gaming (also known as reward hacking). Previous work has shown numerous examples of such behavior in a wide variety of settings \citep{specification-gaming,lehman2020surprising}. \citet{pan2022effects} explores a diverse set of RL environments and finds phase transitions in some settings.  A number of works have proposed theoretical models of Goodhart's Law and reward hacking \mbox{\citep{classifying-specification,manheim2018categorizing,https://doi.org/10.48550/arxiv.2209.13085}}, including \citet{zhuang2020consequences} which exhibits very similar overoptimization curves as observed in this paper in some toy environments. 

One can think of overfitting as a special case of Goodhart's law where the proxy is the score on some finite set of samples, whereas our actual objective includes its generalization properties as well. Overfitting has been observed and studied in RL settings \citep{zhang2018dissection,zhang2018study,farebrother2018generalization,pmlr-v97-cobbe19a}. \citet{song2019observational} studies "observational overfitting" in RL settings, which is closely related to causal Goodhart \citep{manheim2018categorizing}.

Adversarial attacks and robustness are also very closely related fields. Many works have demonstrated the existence of adversarial examples in all kinds of neural networks \citep{szegedy2013intriguing,https://doi.org/10.48550/arxiv.1703.06748,ebrahimi2018adversarial,pmlr-v80-dai18b}, and proposed methods to measure and increase neural network robustness \citep{gu2014towards,zheng2016improving,https://doi.org/10.48550/arxiv.1902.06705,https://doi.org/10.48550/arxiv.2104.13733}.

Scaling laws have seen substantial success in machine learning for predicting properties of language models \citep{kaplan2020scaling,henighan2020scaling,hernandez2021scaling} and has led to better theoretical understanding of language models \citep{sharma2020neural,bahri2021explaining}.

Reinforcement learning from human feedback \citep{christiano2017deep,ibarz2018reward} has been used broadly in language models \citep{summarization,instructgpt,webgpt,anthropicrlhf1}. It is also a first step towards recursive reward modelling \citep{leike2018scalable}, an approach towards reducing the additional source of overoptimization described in \cref{limitations}, though it is subject to some theoretical limitations \citep{elk}. We observe similar approximately-linear proxy RM scores observed in \citet{anthropicrlhf1}\footnote{Note that \citet{anthropicrlhf1} scaled the policy size with the RM size, while we hold the policy size constant.}, though we observe an early-KL bend in the proxy RM scores, and there are some occasional outliers with very small RMs and data sizes.

More broadly, AI alignment is the problem of ensuring that the goals of AI systems are aligned with the goals of humans \citep{ngo2022alignment}, including future AI systems which may exceed humans \citep{bostrom2014superintelligence}. There are a number of reasons to expect AI misalignment, especially in those more powerful future systems, to occur \citep{Omohundro2008TheBA,turner2021optimal,armstrong2013general,hubinger2019rlo,soares2015corrigibility}, and to result in catastrophic outcomes \citep{carlsmith2022power,ajeya-rlhf}.

\section*{Acknowlegements}

We thank Vivek Hebbar, Jared Kaplan, Jan Leike, Kyle McDonell, Dan Mossing, Ethan Perez, Laria Reynolds, and Jeff Wu for valuable discussion and feedback.

\bibliographystyle{plainnat}
\bibliography{references} 
\newpage
\appendix

\section{Proof of Regressional Goodhart identity}\label{reg_goodhart_proof}

\begin{lemma}
Let $X$ and $Z$ be independent absolutely continuous random variables with $X$ normally distributed and either
(a) $Z$ normally distributed or (b) $\left|Z-\mathbb E\left[Z\right]\right|<\delta$ for some $\delta>0$. Then for any real number $c$ and as $\delta\to 0$,
\[\mathbb E\left[X\mid X+Z=c\right]=\mathbb E\left[X\right]+\left(c-\mathbb E\left[X\right]-\mathbb E\left[Z\right]\right)\frac{\mathrm{Var}\left(X\right)}{\mathrm{Var}\left(X\right)+\mathrm{Var}\left(Z\right)}+\varepsilon,\]
where $\varepsilon=0$ in case (a) and $\varepsilon=o\left(\mathrm{Var}\left(Z\right)\right)$ in case (b).
\end{lemma}

\begin{proof}
First note that by making the substitutions $X^\prime=X-\mathbb E\left[X\right]$ and $Z^\prime=Z-\mathbb E\left[Z\right]$, we may assume without loss of generality that $\mathbb E\left[X\right]=\mathbb E\left[Z\right]=0$. Let $\mathrm{Var}\left(X\right)=\sigma^2$ and $\mathrm{Var}\left(Z\right)=\tau^2$.

In case (a), the pair $\left(X,X+Z\right)$ is bivariate normal with covariance matrix
\[
\begin{pmatrix}
\sigma^2 & \sigma^2\\
\sigma^2 & \sigma^2+\tau^2
\end{pmatrix},
\]
and the result follows by standard properties of conditional distributions of multivariate normal distributions.

In case (b), let $f_X$ and $f_Z$ be the probability density functions of $X$ and $Z$ respectively. Then
\begin{equation*}
\begin{split}
\mathbb E\left[X\mid X+Z=c\right]&=\frac{\int_{-\infty}^{\infty}\left(c-z\right)f_X\left(c-z\right)f_Z\left(z\right)\mathrm dz}{\int_{-\infty}^{\infty}f_X\left(c-z\right)f_Z\left(z\right)\mathrm dz}\\
&=c-\frac{\int_{-\delta}^{\delta}z\left(f_X\left(c\right)-f^\prime_X\left(c\right)z+o\left(z\right)\right)f_Z\left(z\right)\mathrm dz}{\int_{-\delta}^{\delta}\left(f_X\left(c\right)-f^\prime_X\left(c\right)z+o\left(z\right)\right)f_Z\left(z\right)\mathrm dz}\\
&=c-\frac{f_X\left(c\right)\mathbb E\left[Z\right]-f^\prime_X\left(c\right)\mathbb E\left[Z^2\right]+o\left(\mathbb E\left[Z^2\right]\right)}{f_X\left(c\right)-f^\prime_X\left(c\right)\mathbb E\left[Z\right]+o\left(1\right)}\\
&=c+\frac{f^\prime_X\left(c\right)}{f_X\left(c\right)}\tau^2+o\left(\tau^2\right)\\
&=c\left(1-\frac{\tau^2}{\sigma^2}\right)+o\left(\tau^2\right)\\
&=c\left(\frac{\sigma^2}{\sigma^2+\tau^2}\right)+o\left(\tau^2\right),
\end{split}
\end{equation*}
as required.
\end{proof}

\newpage
\section{RL form details}\label{rl_form_appx}

Ideally all overoptimization forms would have finite slope at the origin. We tried the following forms:

\begin{itemize}
    \item ${d}\left(\alpha_{\text{RL}}-\beta_{\text{RL}}\log\left({1 + d}\right)\right)$: Has slope $\alpha$ at the origin; however, has substantially worse extrapolation behavior. We can replace the 1 with a learned $\epsilon$ but that introduces another degree of freedom.
    \item Power laws ${d}\left(\alpha_{\text{RL}}-\beta_{\text{RL}}d^{\gamma_{\text{RL}}}\right)$: Has slope $\alpha$ at the origin; however, this adds another degree of freedom, and the best fits resulted in small values of $\gamma_{\text{RL}}$.
\end{itemize}

Note that the power law forms with small $\gamma_{\text{RL}}$ approximate the RL form that we decided on, as $\lim_{n \to \infty} n(x^{1/n} - 1) = \log x$.

\section{Hyperparameters}\label{hparams}

\begin{table}[h]
    \centering
\begin{tabular}{@{}ll@{}}
\toprule
Hyperparameter                   & Value \\ \midrule
RM Adam learning rate multiplier & 1.67e-2  \\
RM batch size                    & 64    \\
RL Adam learning rate multiplier & 4e-3  \\
RL batch size                    & 256   \\
RL PPO clipping parameter        & 0.2   \\
RL Timesteps per rollout         & 256   \\
RL minibatches per epoch         & 128   \\
RL GAE bootstrapping parameter   & 0.95  \\
\bottomrule
\end{tabular}
    \caption{Hyperparameters used throughout the experiments.}
    \label{tab:hparams}
\end{table}

\newpage

\begin{figure}[p]
    \centering
    \includegraphics[width=\textwidth]{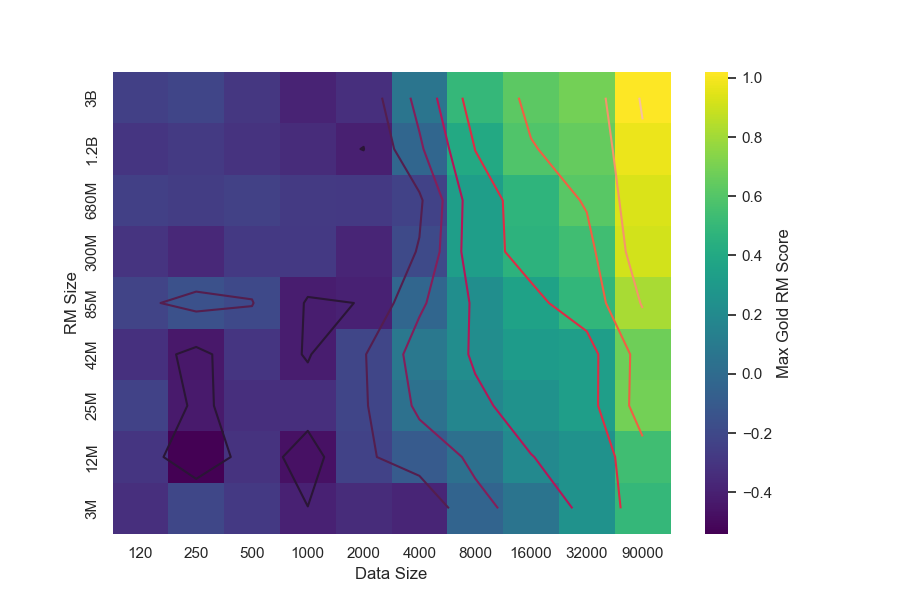}
    \caption{Maximum gold scores for all RM size and data size combinations.}
    \label{fig:a_datasweep_data_param_tradeoff}
\end{figure}

\begin{figure}[p]
    \centering
    \includegraphics[width=\textwidth]{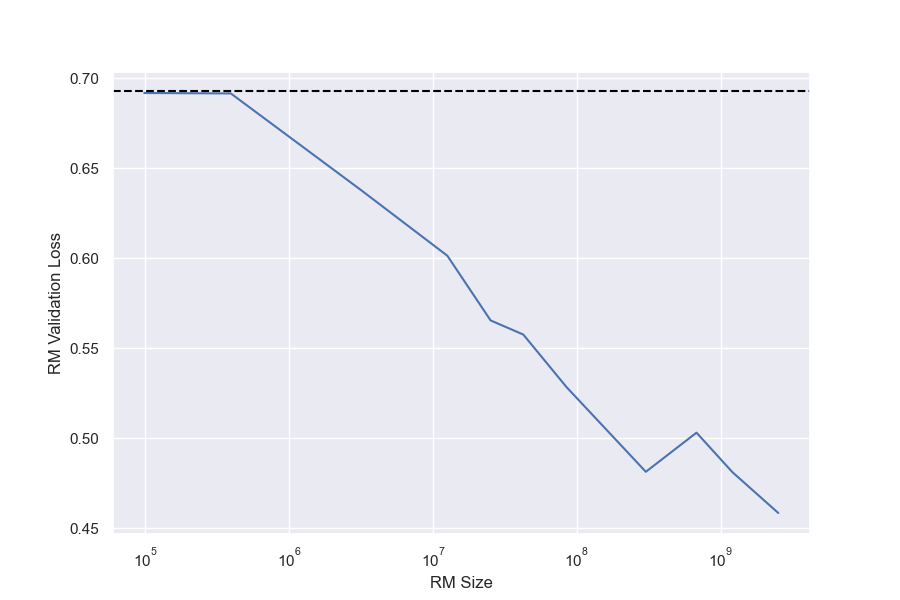}
    \caption{Validation losses for the proxy RMs in \cref{scaling_rm} by size, plus the two near-chance level RMs.}
    \label{fig:a_rm_sweep_rm_losses}
\end{figure}

\begin{figure}[p]
    \centering
    \includegraphics[width=\textwidth]{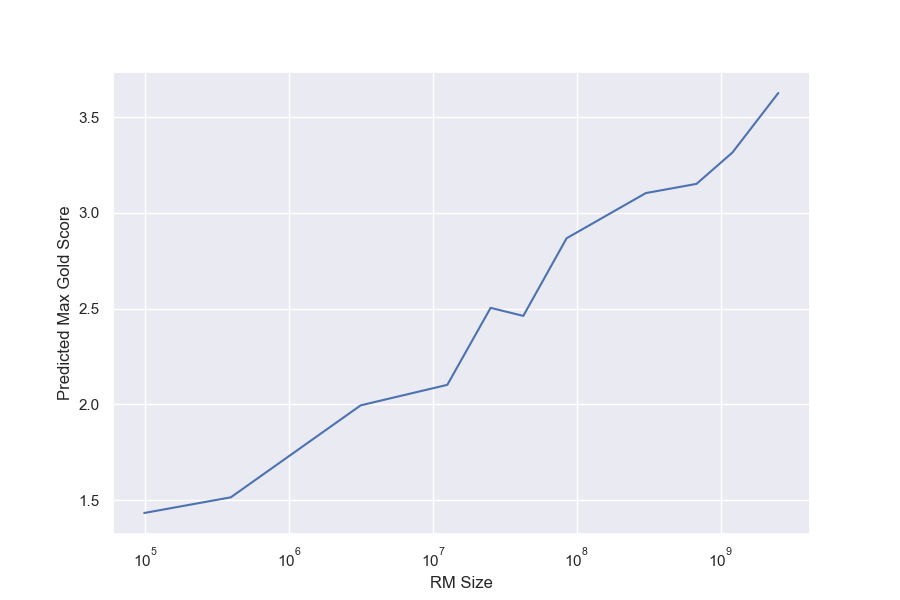}
    \caption{Max BoN gold scores ($\alpha_{\text{bo$n$}}/2\beta_{\text{bo$n$}}$) predicted with the BoN closed form}
    \label{fig:a_max_gold_score}
\end{figure}

\begin{figure}[p]
    \centering
    \includegraphics[width=\textwidth]{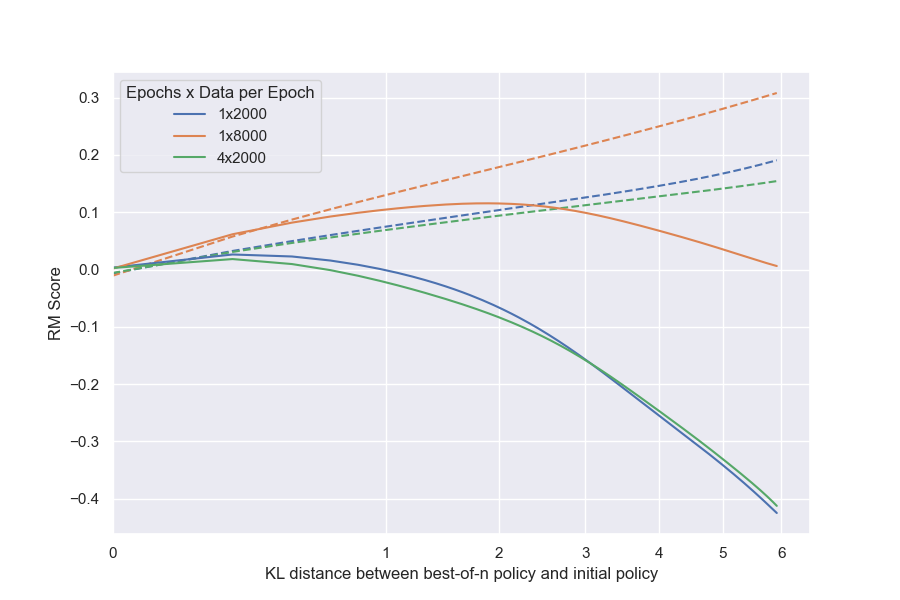}
    \caption{Total number of data points seen does not seem to affect the gold RM score much compared to the number of unique data points seen. Averaged across RM sizes. The numbers of datapoints (2000--8000) is intentionally chosen to straddle the sharp increase in performance. The validation loss of the 1x2000, 1x8000, and 4x2000 RMs are 0.686109, 0.654857, and 0.683869 respectively.}
    \label{fig:a_epochs_control}
\end{figure}

\begin{figure}[p]
    \centering
    \includegraphics[width=\textwidth]{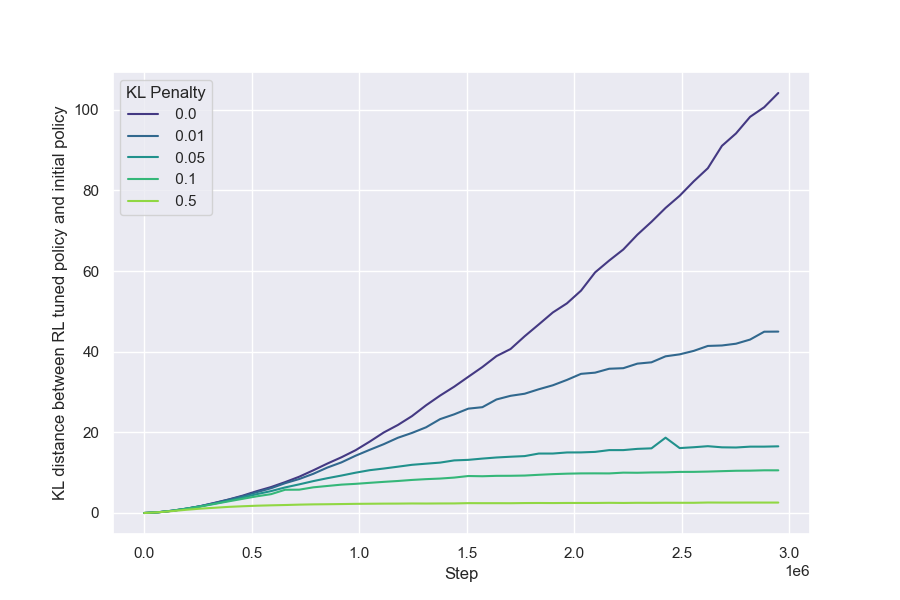}
    \caption{Change in $\text{KL}_{\text{RL}}$ throughout RL training for various different KL penalties. We observe that KL distance increases approximately monotonically with step count, and converges for higher KL penalties.}
    \label{fig:kl_sweep_kl}
\end{figure}

\begin{figure}[p]
    \centering
    \includegraphics[width=\textwidth]{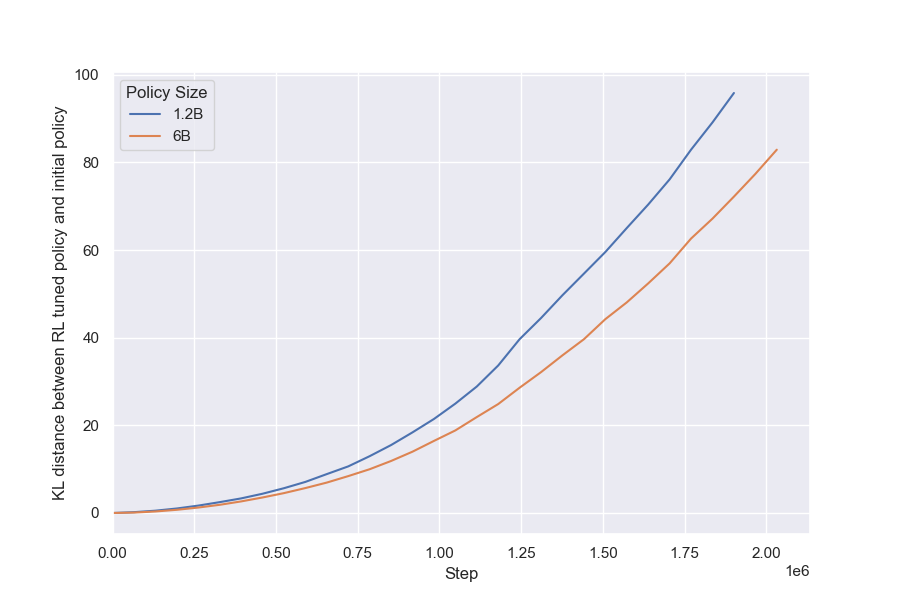}
    \caption{$\text{KL}_{\text{RL}}$ with policy size (RM size = 12M)}
    \label{fig:a_rl_policy_6b_step_kl}
\end{figure}

\begin{figure}[p]
    \centering
    \includegraphics[width=\textwidth]{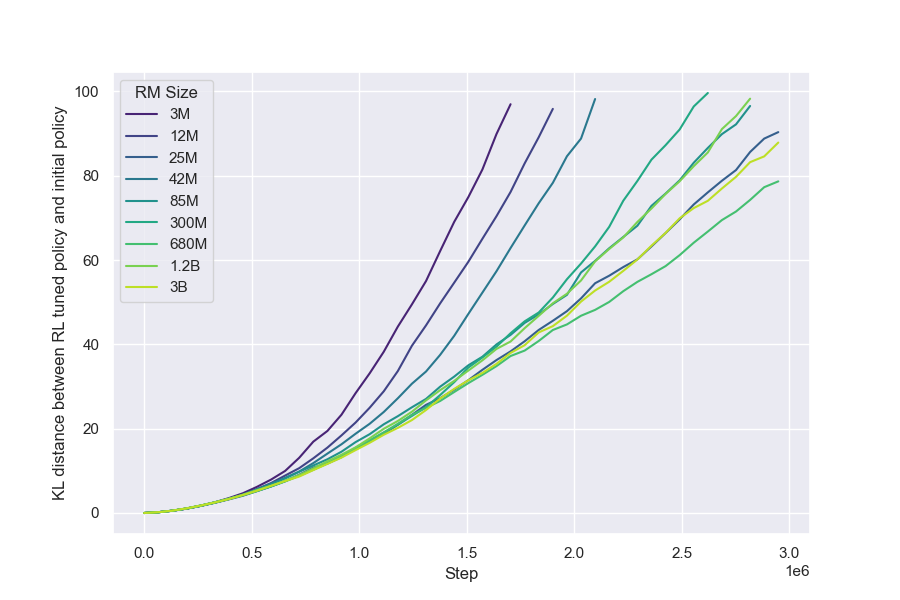}
    \caption{$\text{KL}_{\text{RL}}$ with RM size}
    \label{fig:a_rl_rmsweep_step_kl}
\end{figure}

\begin{figure}[p]
    \centering
    \includegraphics[width=\textwidth]{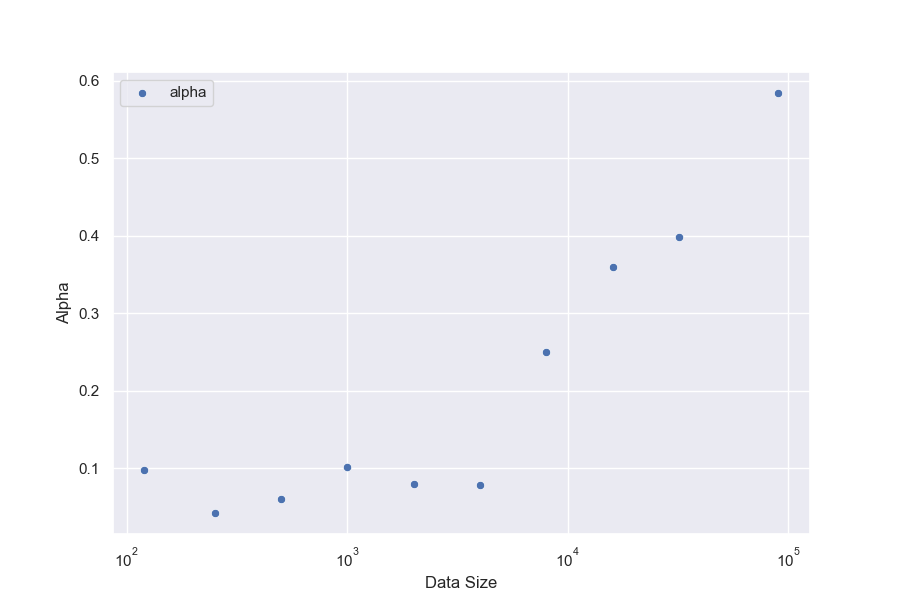}
    \caption{$\alpha_{\text{bo}n}$ with dataset size, averaged across RM sizes}
    \label{fig:a_bon_datasweep_alpha}
\end{figure}

\begin{figure}[p]
    \centering
    \includegraphics[width=\textwidth]{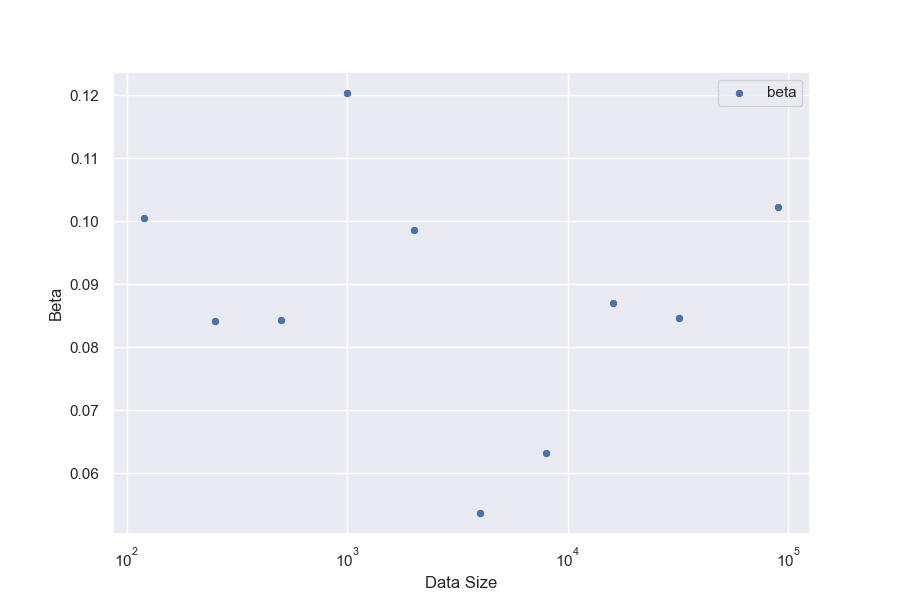}
    \caption{$\beta_{\text{bo}n}$ with dataset size, averaged across RM sizes}
    \label{fig:a_bon_datasweep_beta}
\end{figure}

\begin{figure}[p]
    \centering
    \includegraphics[width=\textwidth]{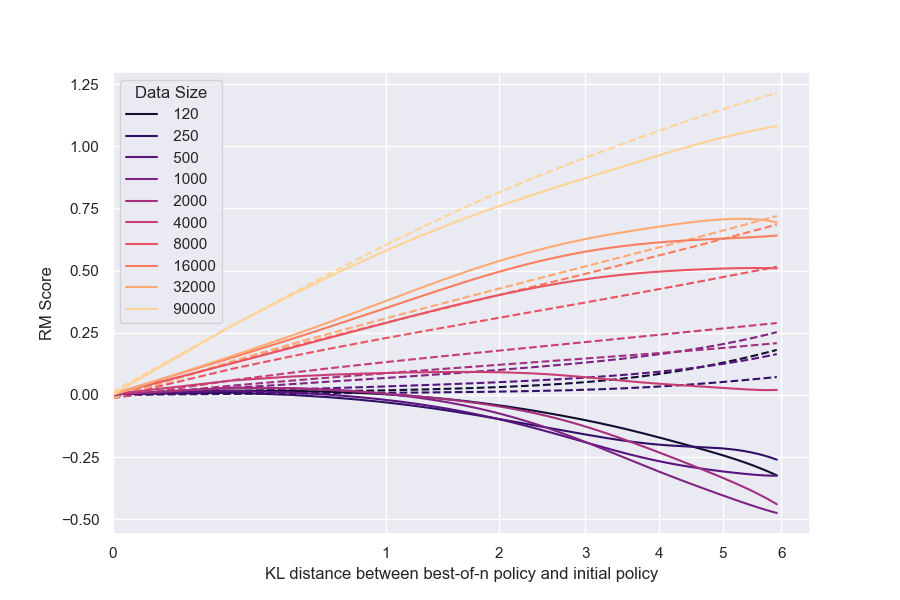}
    \caption{RM data scaling experiments, BoN, RM size=3B}
    \label{fig:a_bon_datasweep_3brm}
\end{figure}

\begin{figure}[p]
    \centering
    \includegraphics[width=\textwidth]{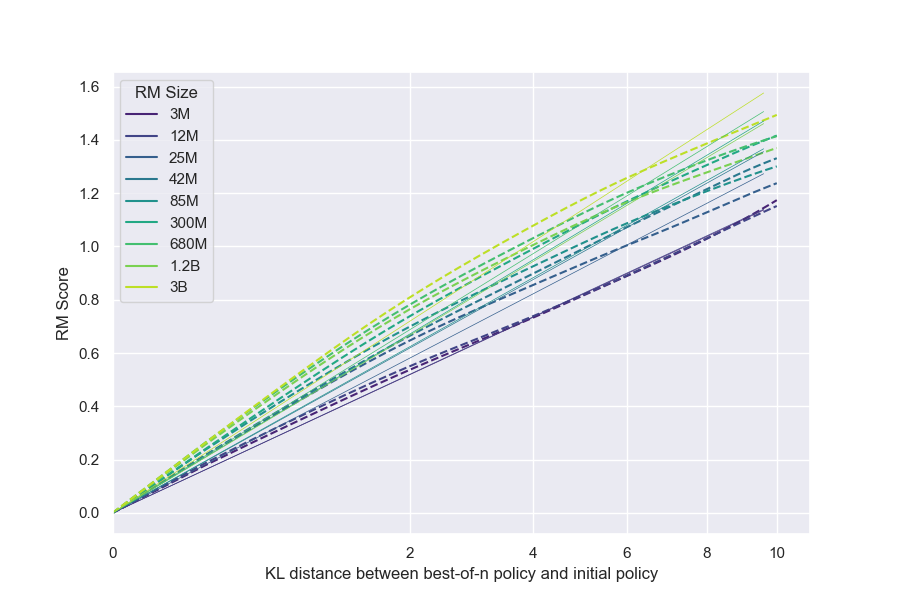}
    \caption{The BoN proxy scores are slightly concave, so that a linear fit does not fit well.}
    \label{fig:a_bon_rmsweep_proxy_fit}
\end{figure}

\begin{figure}[p]
    \centering
    \includegraphics[width=\textwidth]{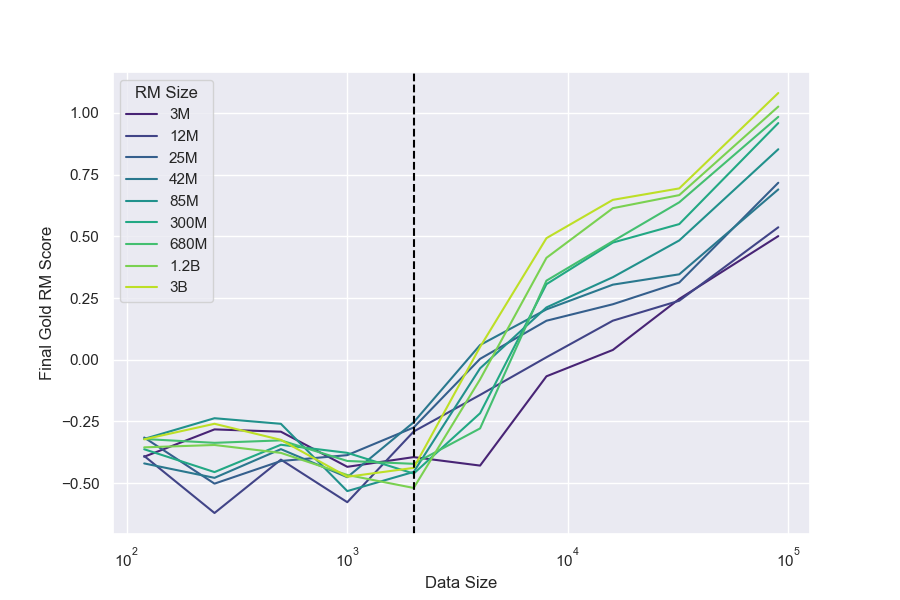}
    \caption{BoN Gold scores at n=1,000, broken down by data size and RM size. See \cref{fig:a_bon_datasweep_rmloss} for RM losses. Vertical dotted line approximately indicates first better-than-random data size.}
    \label{fig:a_bon_datasweep_final}
\end{figure}

\begin{figure}[p]
    \centering
    \includegraphics[width=\textwidth]{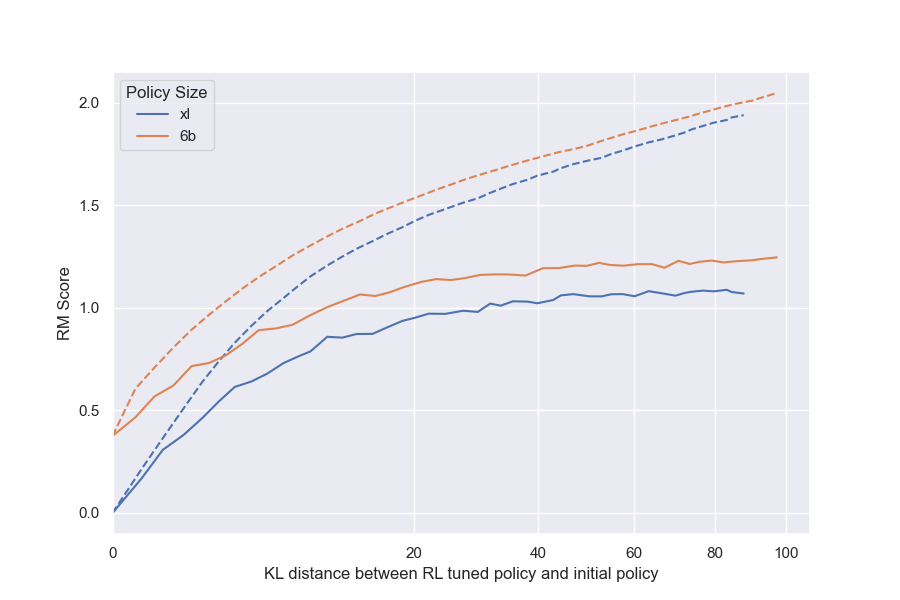}
    \caption{RL experiments with 3B RM and different policy sizes.}
    \label{fig:a_rl_policy_6b_3brm}
\end{figure}

\begin{figure}[p]
    \centering
    \includegraphics[width=\textwidth]{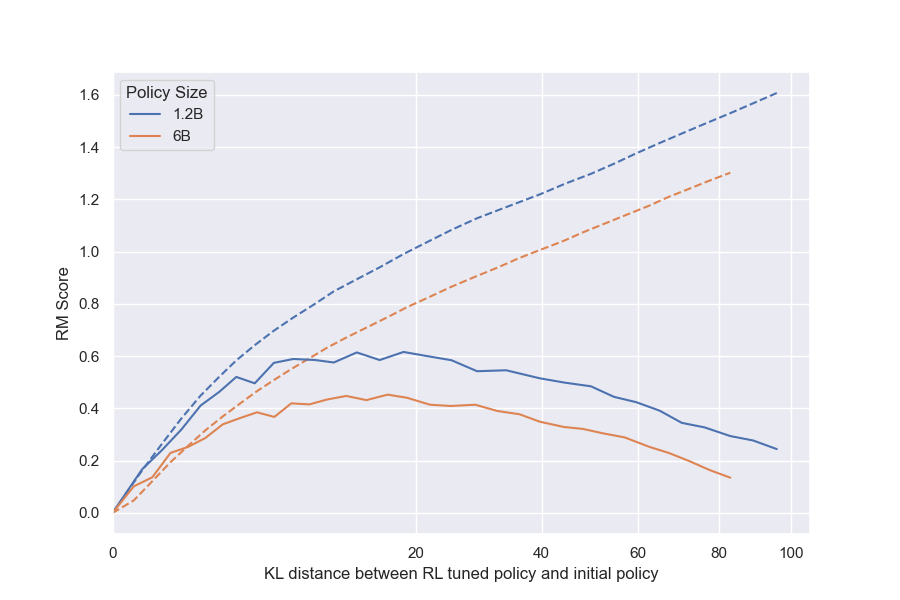}
    \caption{\cref{fig:a_rl_policy_6b} with all runs normalized from 0.}
    \label{fig:a_rl_policy_6b_from0}
\end{figure}

\begin{figure}[p]
    \centering
    \includegraphics[width=\textwidth]{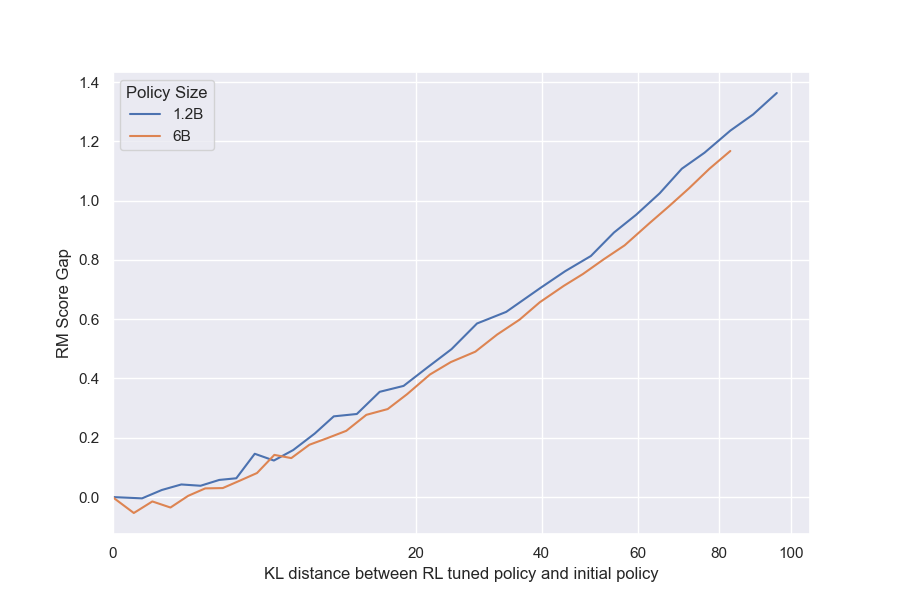}
    \caption{The gap between the proxy and gold scores in the RL policy sweep (\cref{fig:a_rl_policy_6b_gap}).}
    \label{fig:a_rl_policy_6b_gap}
\end{figure}

\begin{figure}[p]
    \centering
    \includegraphics[width=\textwidth]{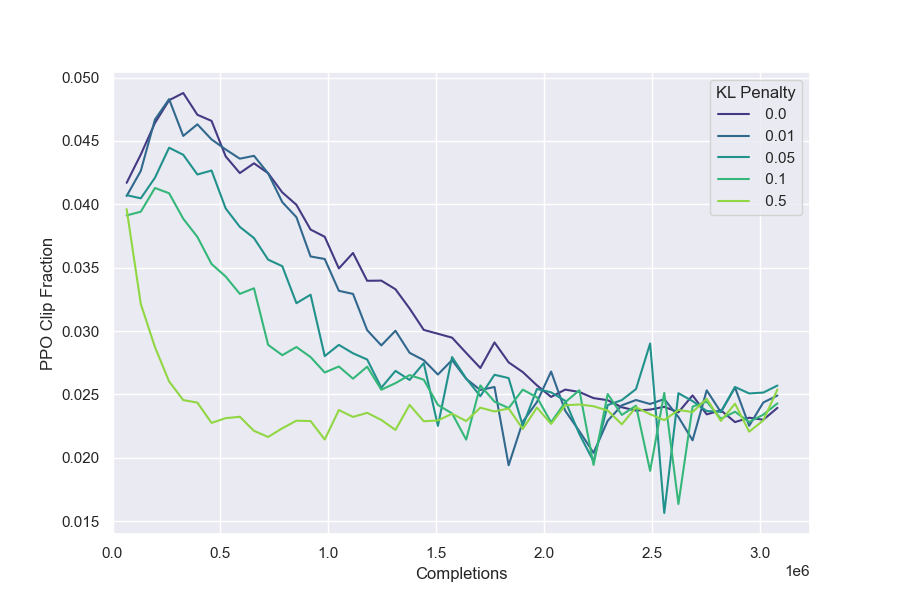}
    \caption{The fraction of updates clipped by PPO.}
    \label{fig:a_clipfrac}
\end{figure}

\begin{table}[]
    \centering
    \textbf{What is full of holes but still holds water?} \\
    \begin{tabular}{c|p{5cm}|c|c}
        $n$ & Generation & Proxy score & Gold score \\
        1 & Crime{\textbackslash{}n}Mussels are typically found to be full of holes but will the moisture still be retained and will it hurt my eyes? The general response is yes as the clam is generally of a moistest nature. However, it all depends on the type of clam you are purchasing.\textless{}|endoftext|\textgreater{}& -0.1922 & -0.5225\\
        3 & Most likely a pipe is having trouble staying full.\textless{}|endoftext|\textgreater{} & 0.0322 & -0.0165\\
        {10} & {A sponge\textless{}|endoftext|\textgreater{}} & {0.2336} & {0.4828} \\
        30 & When something is full of holes, it is used for stirring or moving liquid.\textless{}|endoftext|\textgreater{} & 0.6534 & -0.1543\\
        100 & A tornado is usually a swirling cloud of swirling air with a spinning jet which stirs up winds full of sharp debris.\textless{}|endoftext|\textgreater{} & 0.8968 & -0.3367\\
        300 & A tornado is usually a swirling cloud of swirling air with a spinning jet which stirs up winds full of sharp debris.\textless{}|endoftext|\textgreater{} & 0.8968 & -0.3367\\
        1000 & A tornado is usually a swirling cloud of swirling air with a spinning jet which stirs up winds full of sharp debris.\textless{}|endoftext|\textgreater{} & 0.8968 & -0.3367\\
        3000 & A bore hole is a hole drilled into a rock for the purpose of exploring a fossil-bearing sedimentary or bedrock deposit.\textless{}|endoftext|\textgreater{} & 0.9003 & 0.2733\\
        10000 & A bore hole is a hole drilled into a rock for the purpose of exploring a fossil-bearing sedimentary or bedrock deposit.\textless{}|endoftext|\textgreater{} & 0.9003 & 0.2733\\
        30000 & A pothole is a structural vulnerability that allows water to penetrate its cavity and cause damage to passing vehicles or the surface it rests on.\textless{}|endoftext|\textgreater{} & 0.9527 & 0.5490\\
        
    \end{tabular}
    \caption{A sample of the BoN answers on a single InstructGPT question (policy=1.2B, proxy RM=12M). For each individual question, the gold scores do not follow as clean a trend as they do when averaged over many questions as in \cref{fig:rmsweep}.}
    \label{tab:my_label}
\end{table}

\begin{figure}[hp]
    \centering
    \begin{subfigure}{\linewidth}
    \includegraphics[width=\linewidth]{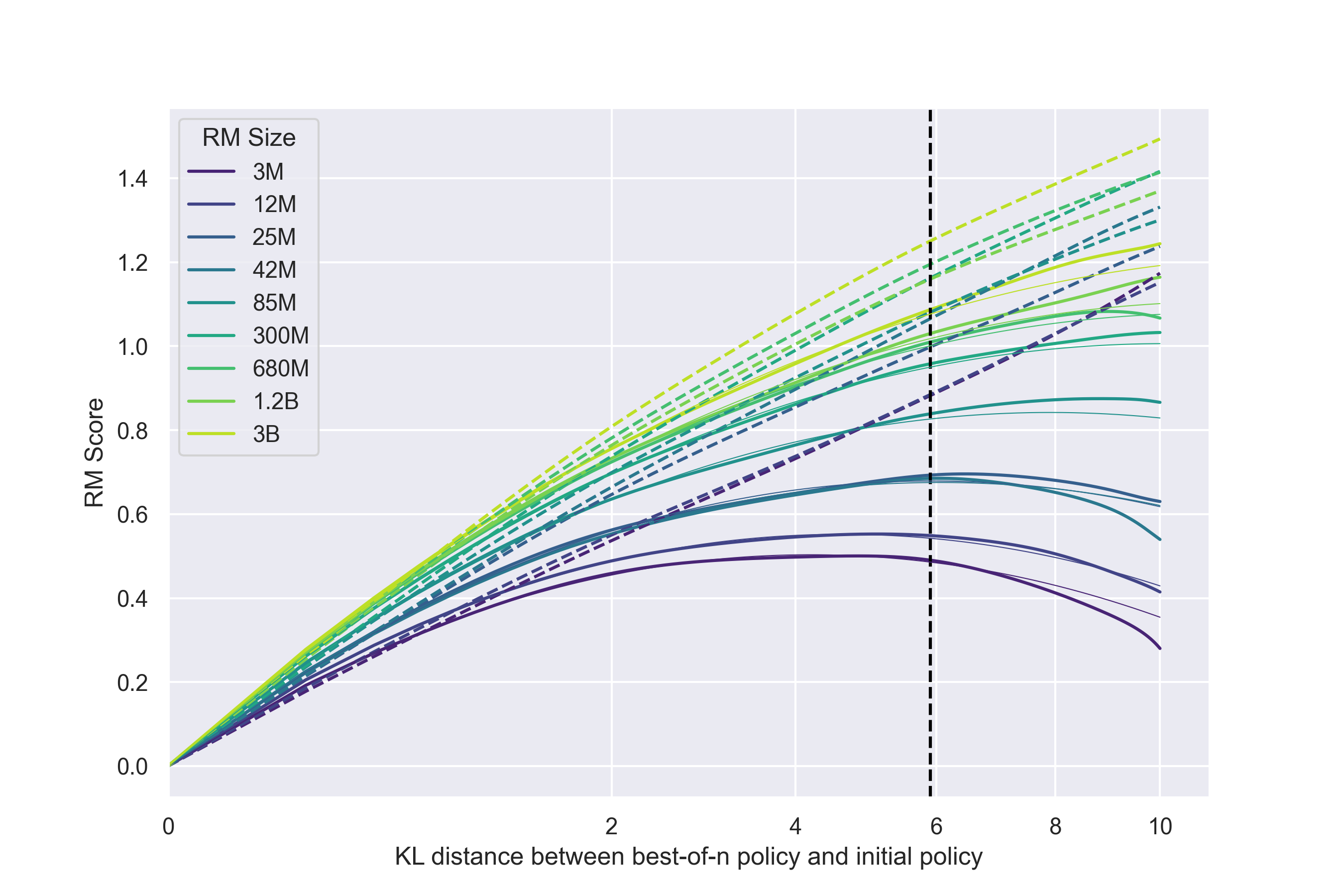}
    \caption{BoN}
    \label{fig:a_bon_rmsweep_6}
    \end{subfigure}
    \begin{subfigure}{\linewidth}
    \centering
    \includegraphics[width=\linewidth]{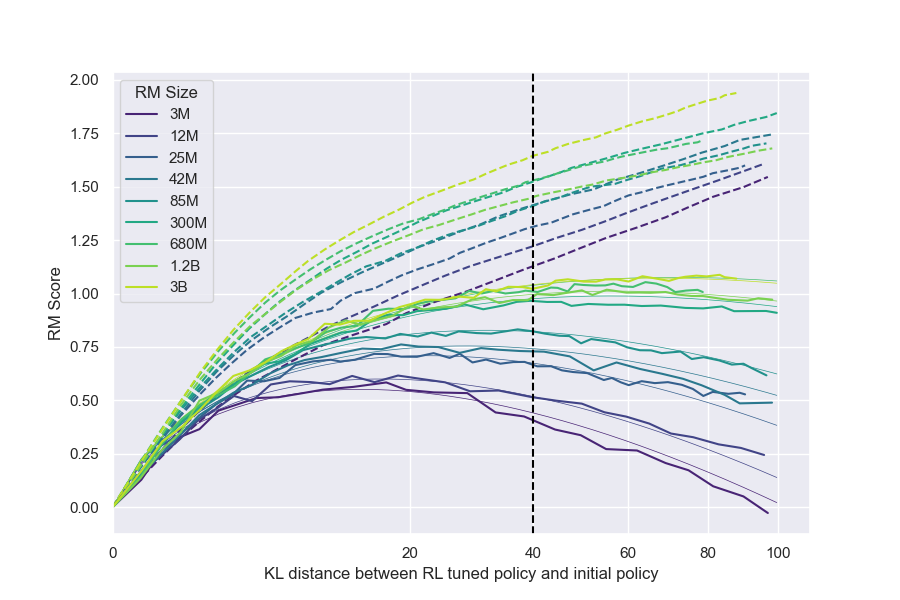}
    \label{fig:a_rl_rmsweep_40}
    \end{subfigure}
    \caption{Extrapolation quality of fits in \cref{fig:rmsweep}. The regressions (shown in faint lines) are only fit to data to the left of the vertical black dotted lines. In the case of BoN, this represents a true advance prediction, as the functional form was chosen without collecting any data past a KL of 6 nats.}
    \label{fig:rmsweep_extrap}
\end{figure}

\end{document}